\documentclass[review]{elsarticle}

\usepackage{lineno,hyperref}
\modulolinenumbers[5]

\usepackage[english]{babel}
\usepackage{booktabs}

\usepackage{ifpdf}

\usepackage{url}
\usepackage{hyperref}

\usepackage{color}

\usepackage{pgf, tikz, pgfplots}
\usetikzlibrary{shapes, arrows, automata}
\usetikzlibrary{calc,hobby,decorations}

\usepackage[cmex10]{amsmath}
\usepackage{amsfonts, amssymb, amsthm}
\usepackage{mathrsfs}

\usepackage{algorithm} 
\usepackage{algorithmic}
\usepackage{enumerate}
\usepackage{multirow}
\usepackage{rotating}
\usepackage{subcaption}
\usepackage{needspace}

\input{mySymbol.sty}

\definecolor{penndarkestblue}{cmyk}{1,0.74,0,0.77}
\definecolor{penndarkerblue}{cmyk}{1,0.74,0,0.70}
\definecolor{pennblue}{cmyk}{0.99,0.66,0,0.57} 
\definecolor{pennlighterblue}{cmyk}{0.98,0.44,0,0.35}
\definecolor{pennlightestblue}{cmyk}{0.38,0.17,0,0.17} 

\definecolor{penndarkestred}{cmyk}{0,1,0.89,0.66}
\definecolor{penndarkerred}{cmyk}{0,1,0.88,0.55}
\definecolor{pennred}{cmyk}{0,1,0.83,0.42} 
\definecolor{pennlighterred}{cmyk}{0,1,0.6,0.24}
\definecolor{pennlightestred}{cmyk}{0,0.43,0.26,0.12} 

\definecolor{penndarkestgreen}{cmyk}{1,0,1,0.68}
\definecolor{penndarkergreen}{cmyk}{1,0,1,0.57}
\definecolor{penngreen}{cmyk}{1,0,1,0.44} 
\definecolor{pennlightergreen}{cmyk}{1,0,1,0.25}
\definecolor{pennlightestgreen}{cmyk}{0.43,0,0.43,0.13}

\definecolor{penndarkestorange}{cmyk}{0,0.65,1,0.49}
\definecolor{penndarkerorange}{cmyk}{0,0.65,1,0.33}
\definecolor{pennorange}{cmyk}{0,0.54,1,0.24} 
\definecolor{pennlighterorange}{cmyk}{0,0.32,1,0.13}
\definecolor{pennlightestorange}{cmyk}{0,0.15,0.46,0.06}
	
\definecolor{penndarkestpurple}{cmyk}{0,1,0.11,0.86}
\definecolor{penndarkerpurple}{cmyk}{0,1,0.13,0.82}
\definecolor{pennpurple}{cmyk}{0,1,0.11,0.71} 
\definecolor{pennlighterpurple}{cmyk}{0,1,0.05,0.46}
\definecolor{pennlightestpurple}{cmyk}{0,0.35,0.02,0.23}
	
\definecolor{pennyellow}{cmyk}{0,0.20,1,0.05} 
\definecolor{pennlightgray1}{cmyk}{0,0,0,0.05}
\definecolor{pennlightgray3}{cmyk}{0.01,0.01,0,0.18}
\definecolor{pennmediumgray1}{cmyk}{0.04,0.03,0,0.31}
\definecolor{pennmediumgray4}{cmyk}{0.08,0.06,0,0.54}
\definecolor{penndarkgray2}{cmyk}{0.09,0.07,0,0.71}
\definecolor{penndarkgray4}{cmyk}{0.1,0.1,0,0.92}

\def\SO3{\mathrm{SO(3)}}

\newtheorem{lemma}{\hspace{0pt}\bf Lemma}

\newtheorem{theorem}{\hspace{0pt}\bf Theorem}
\newtheorem{corollary}{\hspace{0pt}\bf Corollary}

\newtheorem{remark}{\hspace{0pt}\bf Remark}

\newtheorem{definition}{\hspace{0pt}\bf Definition}

\journal{Journal of Signal Processing}

\begin{document}

\begin{frontmatter}

\title{Stability of Graph Convolutional Neural Networks\\ to Stochastic Perturbations}

\author{Zhan Gao$^{\dagger}$\fnref{}, Elvin Isufi$^{\ddagger }$ and Alejandro Ribeiro$^{\dagger}$}
\fntext[]{$^{ \dagger}$Department of Electrical and Systems Engineering, University of Pennsylvania, Philadelphia, PA, USA. Email: $\{$gaozhan,aribeiro$\}$@seas.upenn.edu. $^{\ddagger }$Department of Intelligent Systems, Delft University of Technology, Delft, The Netherlands. Email: e.isufi-1@tudelft.nl}

\begin{abstract}
Graph convolutional neural networks (GCNNs) are nonlinear processing tools to learn representations from network data. A key property of GCNNs is their stability to graph perturbations. Current analysis considers deterministic perturbations but fails to provide relevant insights when topological changes are random. This paper investigates the stability of GCNNs to stochastic graph perturbations induced by link losses. In particular, it proves the expected output difference between the GCNN over random perturbed graphs and the GCNN over the nominal graph is upper bounded by a factor that is linear in the link loss probability. We perform the stability analysis in the graph spectral domain such that the result holds uniformly for any graph. This result also shows the role of the nonlinearity and the architecture width and depth, and allows identifying handle to improve the GCNN robustness. Numerical simulations on source localization and robot swarm control corroborate our theoretical findings.
\end{abstract}


\end{frontmatter}


\section{Introduction} \label{sec:intro}

Graph convolutional neural networks (GCNNs) \cite{Defferrard2016, Fernando2019, Wu2019} have shown remarkable success in a variety of network data applications including recommender systems \cite{Ying2018, Wu20192}, wireless communications \cite{eisen2020optimal, gao2020resource}, and distributed agent control \cite{tolstaya2020learning} among others. One of the key properties of GCNNs for this success is their stability to perturbations in the underlying topology \cite{gama2020graphs}. The latter plays a crucial role when deploying the GCNN on a topology that changes (slightly) from the nominal one used during training. This is a typical scenario encountered in aforementioned applications and allows the GCNN to maintain its performance.

Characterizing the stability of the GCNN to topological perturbations allows identifying handle to improve its robustness. This direction has recently attracted attention in the community \cite{zou2020graph, gama2019diffusion, levie2019transferability, gama2020stability, ruiz2020grapha, ruiz2020graphb}. The work in \cite{zou2020graph} considered the stability of graph scattering transforms --non-trainable graph neural networks using graph wavelet filters \cite{hammond2011wavelets}-- to perturbations of the underlying graph. Subsequently, authors in \cite{gama2019diffusion} specialized the latter to diffusion wavelets \cite{coifman2006diffusion} and characterized its stability to graph deformations measured by the diffusion distance \cite{coifman2006diffusion1}. For trainable architectures, the work in \cite{levie2019transferability} established that graph convolutional neural networks are stable to topological perturbations. It proved GCNNs yield similar representations on graphs describing the same phenomenon. Differently, authors in \cite{gama2020stability, ruiz2020grapha} investigated the stability of GCNNs under relative perturbations, i.e., perturbations that tie to the underlying graph. Their results showed GCNNs can be both stable to perturbations and discriminative in the high graph frequency information. The work in \cite{ruiz2020graphb} extends these stability results to the graphon neural network, where the graphon is the limit of convergent graph sequences.

The robustness of the GCNN has also been studied under targeted attacks on the underlying graph. The work in \cite{zugner2018adversarial} focused on designing adversarial attacks on the graph edges and the node signals to misclassify the target node label, and authors in \cite{dai2018adversarial} designed adversarial attacks via reinforcement learning on both node-level and graph-level classification tasks. In parallel, the work in \cite{bojchevski2019certifiable} investigated the robustness of the GCNN under adversarial attacks on a subset of graph edges, and showed nodes will not change the learned label under these attacks in node classification tasks. Similar robustness results have also been established in \cite{zugner2019certifiable} under adversarial attacks on the node signals.

Altogether, the above results discuss the GCNN stability to deterministic graph perturbations, where perturbation sizes are assumed small. However, the graph topology often changes randomly resulting in large stochastic perturbations \cite{Isufi17, Zhan2020}. This is the case when the GCNN is employed distributively on physical networks, where each node computes its output by communicating with its direct neighbors \cite{Zou2013, Shuman2018, gao2020wide}. For example, in a wireless sensor network where nodes are sensors and edges are communication links, the latter may break randomly due to channel fading effects leading to stochastic communication graphs \cite{structural2004, kar2008sensor}. Other applications in which the GCNN has to cope with random topologies involve distributed robot coordination \cite{antonelli2014decentralized}, smart grids \cite{Gungor2010}, and road networks \cite{deng2016latent}. In these instances, the GCNN operates over a sequence of random graphs, which leads to a random output. The work in \cite{Zhan2020} approached this setting and proposed stochastic graph neural networks (SGNNs) to account for the graph randomness during training. That is, the SGNN is not trained anymore over the nominal deterministic graph but rather over random graphs. Their results showed the SGNN can learn representations that account for the graph randomness. This procedure, however, requires the training to be averaged over different graph realizations which may be computationally inefficient. To improve on the latter, we aim here to characterize the stability of the GCNN to stochastic perturbations --architectures trained over the nominal graph and deployed over random perturbed graphs. This gives insights on the impact of graph stochasticity and allows identifying handle to control it. However, this stability analysis differs from that conducted in earlier works because the topological randomness could induce large graph perturbations and needs to be accounted in a statistical fashion. The GCNN will operate over a sequence of random graphs rather than a single graph and in turn, this requires handling multiple perturbed topologies rather than a fixed one.

In this paper, we investigate the stability of graph convolutional neural networks to stochastic topological perturbations induced by random link losses. By developing the new methodology that avoids the small perturbation assumption in the deterministic setting \cite{gama2020stability}, we derive results that account for the successions of statistically perturbed graphs throughout the GCNN architecture and characterize the explicit impact of graph stochasticity on the architecture behavior. We find out the GCNN is Lipschitz stable to stochastic perturbations in the underlying graph and maintains its performance when the link losses are mild. Our detailed contributions can be summarized as follows.

\smallskip
\begin{enumerate}[(i)]

\item \emph{Stability of graph filters (Section \ref{sec:GraphStability})}: We prove the expected output difference of the graph filter induced by stochastic perturbations is upper bounded proportionally to the link loss probability [Thm. \ref{theorem:filterStability}]. To conduct such analysis, we develop the concept of generalized graph filter frequency response, which allows for the spectral analysis over a sequence of random graphs. We also defined the Lipschitz gradient and put forth the generalized integral Lipschitz condition for tractable mathematical analysis of the GCNN in the stochastic setting. These allow us to claim the stability result independently on the underlying topology.

\item \emph{Stability of graph convolutional neural networks (Section \ref{sec:GCNNStability})}: Leveraging the stability result of the graph filter, we prove the GCNN is stable to stochastic perturbations with a factor proportional to the link loss probability [Thm. \ref{theorem:GNNstability}]. The result indicates also the explicit impact of the filter, nonlinearity, and architecture width and depth on the stability. In particular, a wider and deeper GCNN decreases the stability while improving the performance, indicating a trade-off between such two factors.

\end{enumerate}

These theoretical contributions are corroborated with numerical results on distributed source localization and robot swarm control in Section \ref{sec:numericalExperiments}. The paper conclusions are drawn in Section \ref{sec:conclusions}. All proofs are collected in the appendix.

\section{Stochastic Perturbations on Graph Convolutional Neural Networks} \label{sec:GNNStochasticPerturbations}

Consider an undirected graph $\ccalG=(\ccalV, \ccalE)$ with node set $\ccalV = \{ 1, \ldots, n \}$ and edge set $\ccalE = \{ (i,j) \}$. The graph is represented by the symmetric graph shift operator matrix (e.g., adjacency or Laplacian) $\bbS$ with entry $[\bbS]_{ij} \neq 0$ if and only if $(i,j)\in \ccalE$ or $i=j$. 
The data supported on the vertices of graph $\ccalG$ forms the graph signal and it is a vector $\bbx \in \mathbb{R}^{n}$ with $i$th entry $[\bbx]_i$ assigned to node $i$. The graph encodes similarities between signal values implying that two entries $[\bbx]_i$ and $[\bbx]_j$ are related if there exists an edge between the respective nodes \cite{Sandry2013, Gavili2017, ortega2018}. For instance, in a wireless sensor network, where nodes are sensors and edges are communication links, the graph signal may be sensor measurements.

We are interested in learning representations for the graph and the graph signal. Specifically, given a training set $\ccalT = \{ (\bbx, \bby) \}$ with $\bbx$ the graph signal and $\bby$ the target representation, our goal is to extract an associate representation $\bby$ from an unseen signal $\bbx \notin \ccalT$. This task can be addressed successfully with graph convolutional neural networks (GCNNs), which are information processing architectures that exploit the coupling between the graph signal $\bbx$ and the underlying graph $\ccalG$. GCNNs consist of cascaded layers with each layer comprising a bank of graph filters and a nonlinearity. At layer $\ell$, the GCNN has $F$ input features $\{ \bbx^g_\ell \}_{g=1}^{F}$ processed by $F^2$ graph filters to generate the convolutional features  
\begin{equation}\label{eq:graphFilter}
\bbu^{fg}_\ell = \bbH_{\ell}^{fg}(\bbS) \bbx^g_{\ell-1}:= \sum_{k=0}^K h_{\ell k}^{fg} \bbS^k \bbx_{\ell-1}^g~\text{for}~f,g=1,\ldots,F 
\end{equation}
where $\bbH_{\ell}^{fg}(\bbS) := \sum_{k=0}^K h_{\ell k}^{fg} \bbS^k$ is the graph filter with coefficients $\{ h_{\ell k}^{fg} \}_{k=0}^K$ that processes the $g$th input $\bbx_{\ell-1}^g$ to yield the $f$th output $\bbu_\ell^{fg}$. The graph filter in \eqref{eq:graphFilter} is a polynomial of the graph shift operator $\bbS$ and aggregates signal information from a neighborhood up to radius $K$ at each node, such that the generated features capture a complete picture of the signal-network coupling in each layer. The convolutional features $\bbu_\ell^{fg}$ are subsequently aggregated over the input index $g$ and fed into the nonlinearity $\sigma(\cdot)$ to produce the $F$ output features of layer $\ell$
\begin{equation}\label{eq:GNNArchi}
\bbx^f_\ell = \sigma_\ell \Big( \sum_{g=1}^{F}\bbu_{\ell}^{fg}\Big)~\text{for}~ f=1,\ldots,F.
\end{equation}
These output features are propagated down the cascade until obtaining the output features of the last layer $L$. Without loss of generality, we assume a single input $\bbx = \bbx_0^1$ and a single output $\bbx_L^1$, and define the nonlinear GCNN map $\bbPhi(\bbx; \bbS, \ccalH) = \bbx_L^1$ where parameters $\ccalH = \{ h_{\ell k}^{fg} \}$ collects all filter coefficients.

The GCNN defined by the propagation rule \eqref{eq:graphFilter}-\eqref{eq:GNNArchi} is readily distributable. I.e., it can be implemented at each node by requiring communication only with the immediate neighbors. Specifically, the signal shift $\bbS \bbx$ in \eqref{eq:graphFilter} can be computed by local information exchanges between nodes as
\begin{equation}\label{eq:GSODistributed}
[\bbS \bbx]_i = \sum_{j=1}^n [\bbS]_{ij}[\bbx]_j = \sum_{(i,j) \in \ccalE} [\bbS]_{ij} [\bbx]_j~\text{for}~ i=1,\ldots,n.
\end{equation}
Likewise, the $k$-shifted operation $\bbS^k \bbx$ repeats this process $k$ times and can be computed distributively through $k$ local node exchanges as $\bbS(\bbS^{k-1}\bbx)$ \cite{Coutino2019}. Therefore, the graph filter in \eqref{eq:graphFilter} can be implemented distributively and since the nonlinearity in \eqref{eq:GNNArchi} is pointwise, the overall GCNN results into a distributed processing architecture.

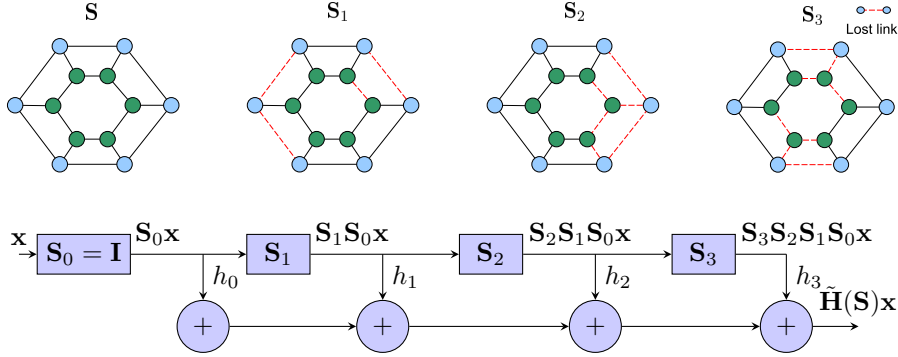
\begin{figure*}[t]
\centering
\includegraphics [width = 0.205\linewidth]
                 {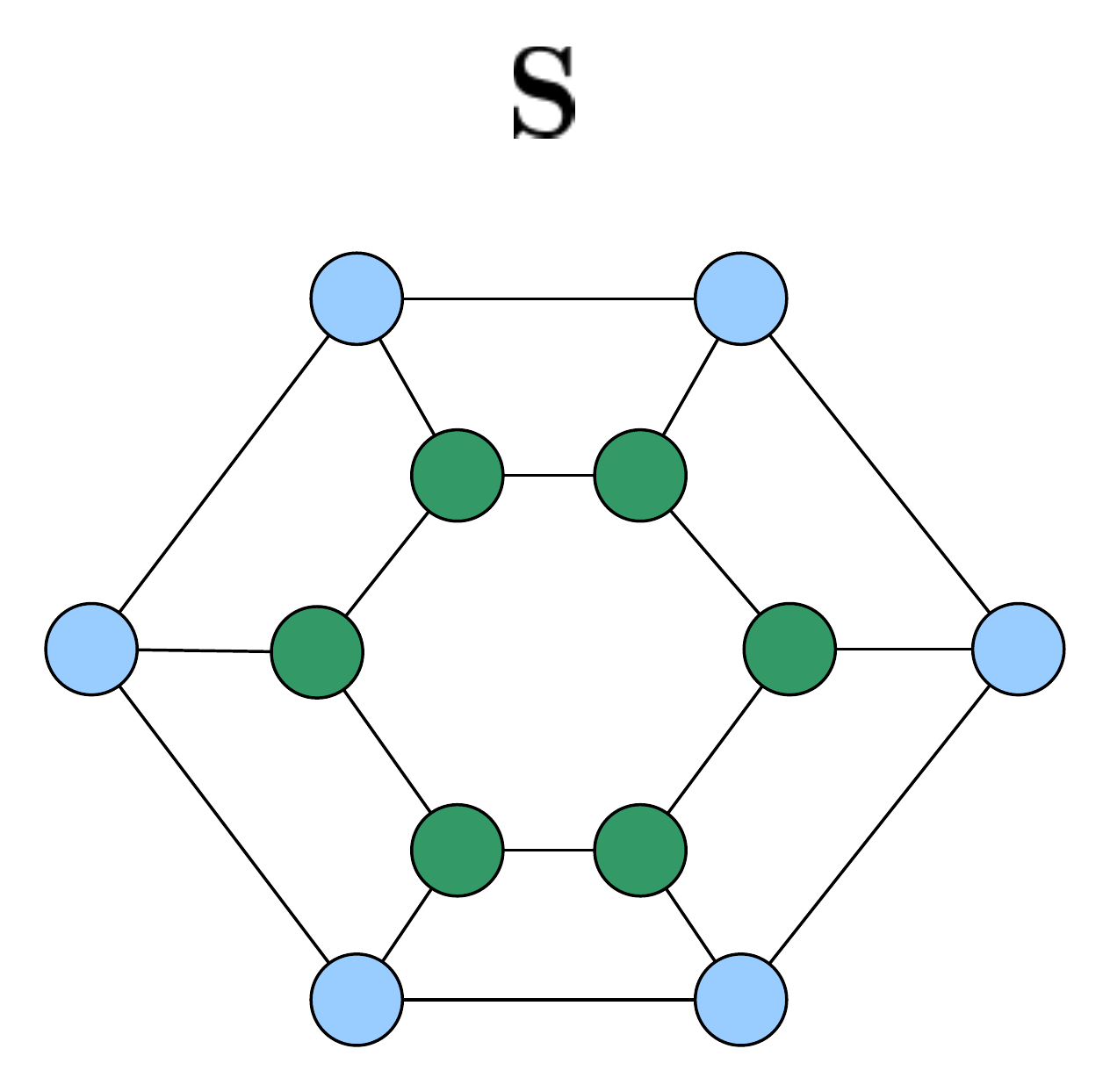}\qquad
\includegraphics [width = 0.205\linewidth]
                 {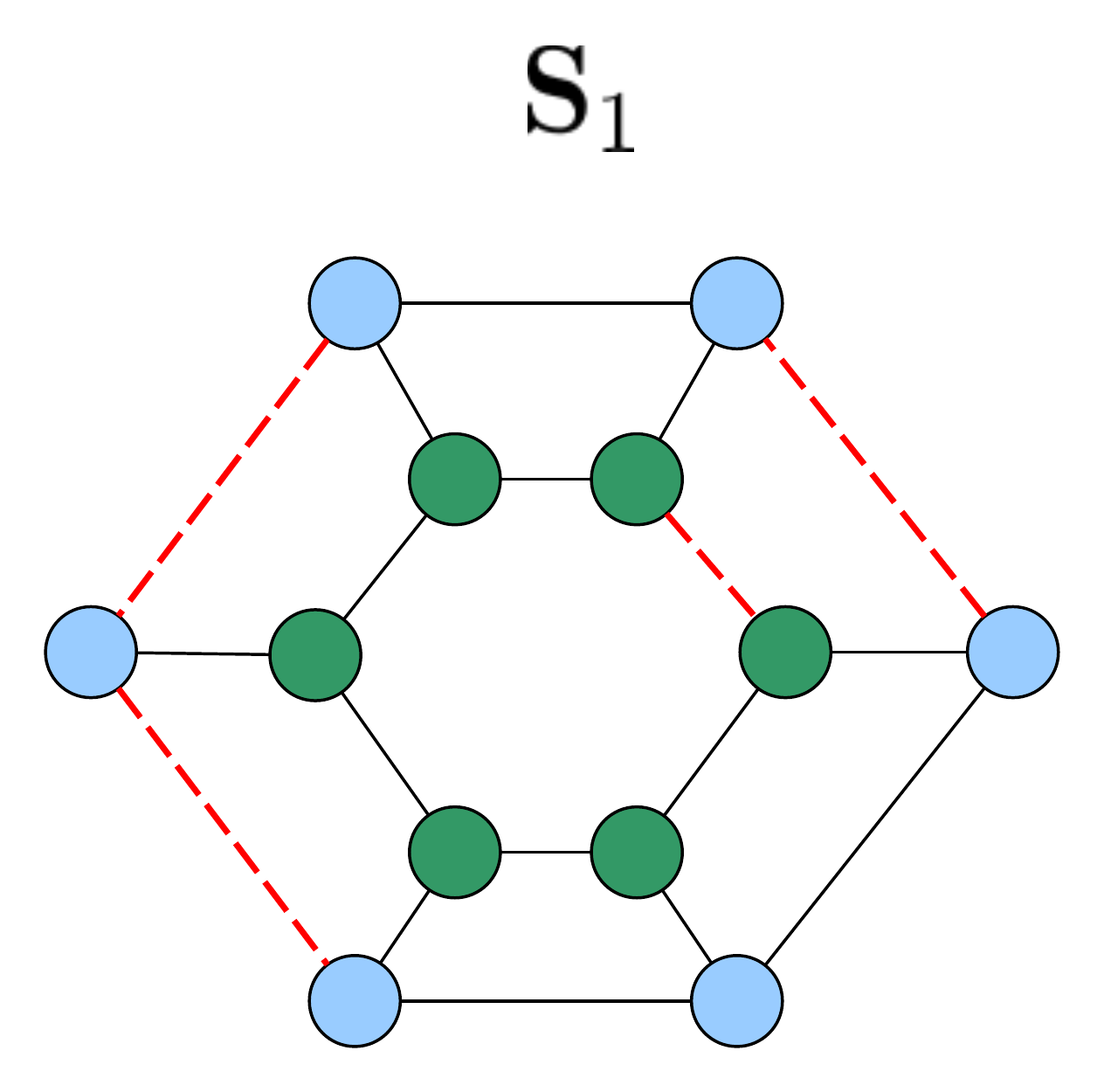}\qquad
\includegraphics [width = 0.205\linewidth]
                 {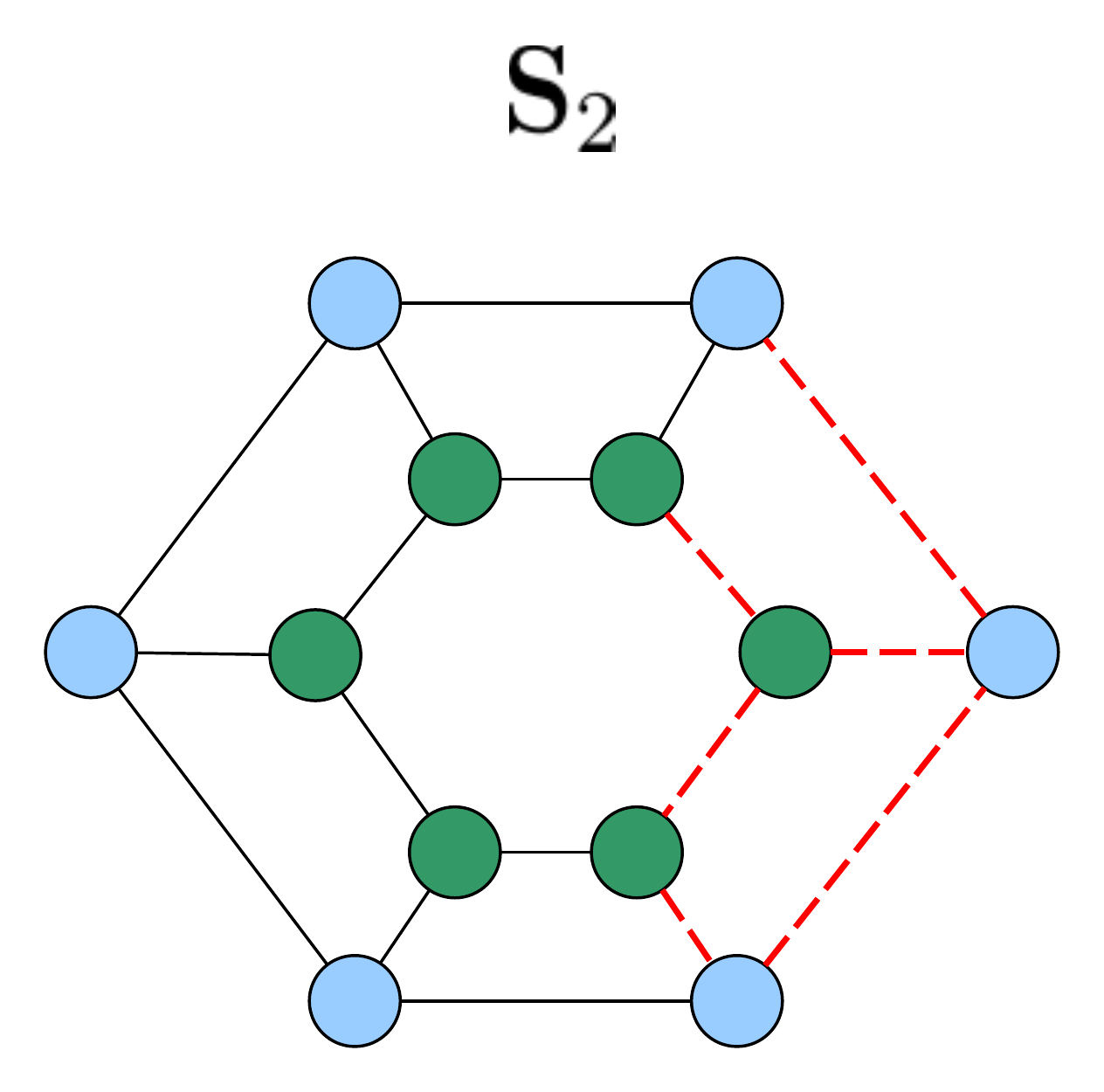}\qquad
\includegraphics [width = 0.205\linewidth]
                 {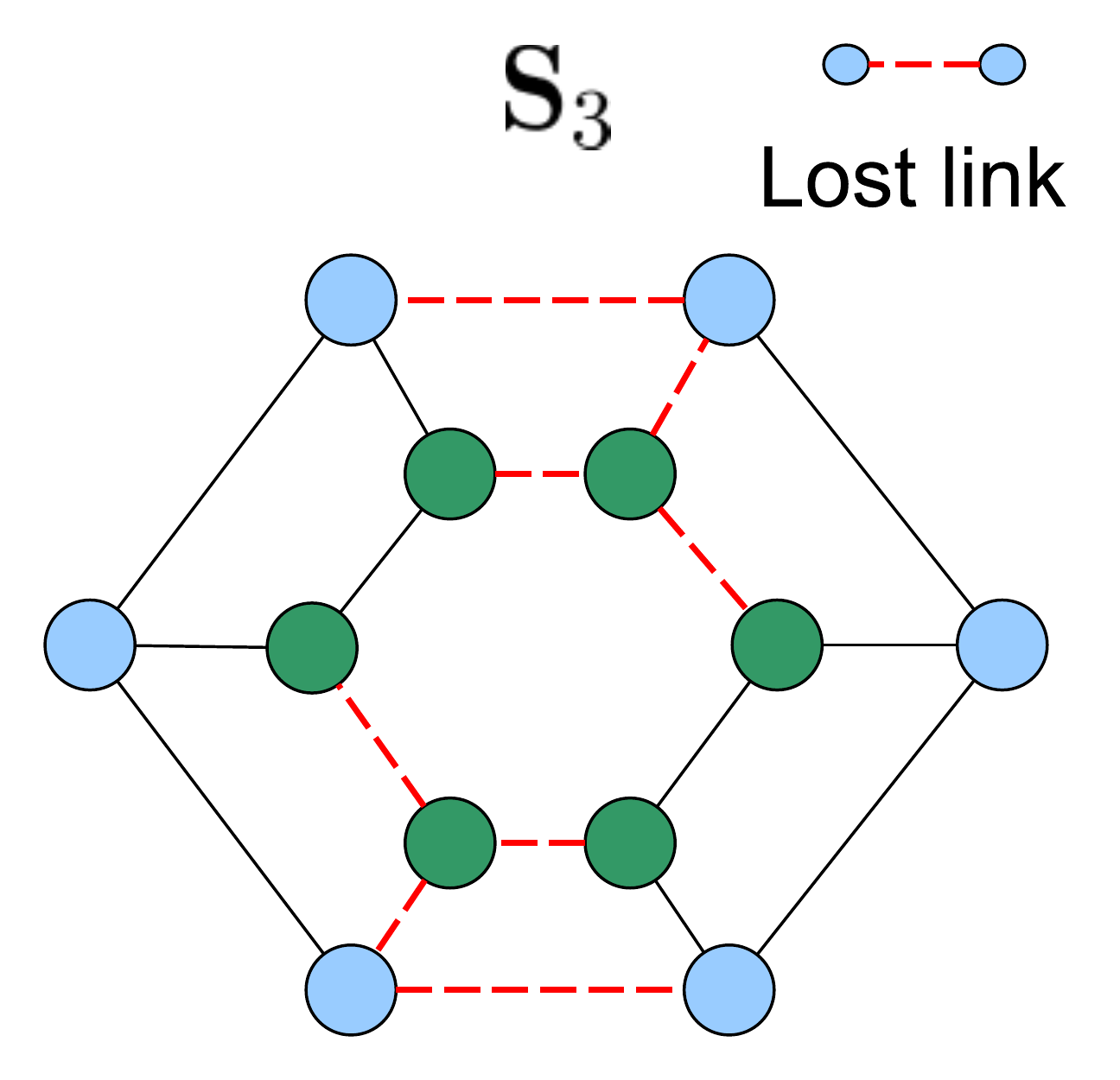} \\ \bigskip

\def \thisplotscale {1.2}
\def \unit {\thisplotscale cm}

\tikzstyle {Phi} = [rectangle,
                    thin,
                    minimum width = 0.7*\unit,
                    minimum height = \sumshift*\unit,
                    anchor = west,
                    draw,
                    fill = blue!20]

\tikzstyle {sum} = [circle,
                    thin,
                    minimum width  = 0.3*\unit,
                    minimum height = 0.3*\unit,
                    anchor = center,
                    draw,
                    fill = blue!20]

\def \deltax {2.0}
\def \deltay {0.8}
\def \sumshift {0.4}

\begin{tikzpicture}[x = 1*\unit, y = 1*\unit]

\node (first) [] {};

\path (first) ++ (0.15*\deltax, 0) node (0) [Phi] {$\bbS_0 = \bbI$};
\path (0)     ++ (0.9*\deltax, 0) node (1) [Phi] {$\bbS_1$};
\path (1)     ++ (1.0*\deltax, 0) node (2) [Phi] {$\bbS_2$};
\path (2)     ++ (1.0*\deltax, 0) node (3) [Phi] {$\bbS_3$};

\path (3.east) ++ (0.7*\sumshift*\deltax, 0) node [anchor=west] (last) [] {};

\path[-stealth] (first) edge [above            ] node {$\bbx~~$}               (0);	
\path[-stealth] (0)     edge [above, near start] node {$\! \bbS_0\bbx$}   (1);	
\path[-stealth] (1)     edge [above, near start] node {$\!\!~~ \bbS_1\bbS_0\bbx$} (2);	
\path[-stealth] (2)     edge [above, near start] node {$\!~~~~~ \bbS_2\bbS_1\bbS_0\bbx$} (3);\path[-]        (3)     edge [above, near end  ] node {$\!~~~~~~ ~~\bbS_{3}\bbS_2\bbS_1\bbS_0\bbx$} (last);				
\path (0.east) ++ (\sumshift*\deltax, -\deltay) node (sum0) [sum] {$+$};
\path (1.east) ++ (\sumshift*\deltax, -\deltay) node (sum1) [sum] {$+$};
\path (2.east) ++ (\sumshift*\deltax, -\deltay) node (sum2) [sum] {$+$};
\path (3.east) ++ (0.7*\sumshift*\deltax, -\deltay) node (sum3) [sum] {$+$};

\path[-stealth, draw] (sum0 |- 0) --node[right] {$ h_0 $} (sum0); 	
\path[-stealth, draw] (sum1 |- 1) --node[right] {$ h_1 $} (sum1);	
\path[-stealth, draw] (sum2 |- 2) --node[right] {$ h_2 $} (sum2);	
\path[-stealth, draw] (sum3 |- 3) --node[right] {$ h_3 $} (sum3);	

\path[-stealth, draw] (sum0) -- (sum1);	
\path[-stealth, draw] (sum1) -- (sum2);	
\path[-stealth, draw] (sum2) -- (sum3);	

\path[-stealth] (sum3) edge [above] node
                {~~$~~~\tilde{\bbH}(\bbS)\bbx$} ++ (0.4*\deltax, 0);

\end{tikzpicture}
\caption{The graph filter $\tilde{\bbH}(\bbS)$ of order $K=3$ run over random graphs. In the distributed implementation, the input signal $\bbx$ is shifted over a chain of random graph shift operators $\bbS_1, \bbS_2, \bbS_3$ and the shifted signals are aggregated to generate the filter output.} 
\label{fig.evRecMain}
\end{figure*}

In a distributed setting, the GCNN needs to cope with random topological perturbations in the graph structure. This is the case when the communication links fall with a certain probability in a wireless sensor network due to channel effects or when the GCNN is susceptible to adversarial attacks. These topological perturbations affect the signal shifts in \eqref{eq:graphFilter} and change the filtering operation to\footnote{We drop the subscript $\ell$ and the superscripts $fg$ of $\tilde{\bbu}_\ell^{fg}$ to simplify notation.}
\begin{equation}\label{eq:randomGraphFilter}
\tilde{\bbu} = \tilde{\bbH}(\bbS) \bbx = \sum_{k=0}^K h_k \bbS_k \cdots \bbS_1\bbS_0 \bbx
\end{equation}
where $\{ \bbS_k \}_{k=1}^K$ are random variants of the underlying shift operator $\bbS$ and $\bbS_0 = \bbI$ is the identity matrix---see Fig. \ref{fig.evRecMain}. If these stochastic perturbations are substantial or uncontrolled, the resulting GCNN output run with random graph filters [cf. \eqref{eq:randomGraphFilter}] will deviate substantially from the learned representation w.r.t. the nominal graph. Hence, it is necessary to quantify their effects. 

This paper aims to shed light on the effects stochastic perturbations have on the GCNN output and, ultimately, identify handle to control them. To be more precise, we consider the graph is affected by random edge drops in each shift and conduct a statistical analysis of the GCNN stability. Using these results, we characterize the role played by different architecture components such as the filter shape, nonlinearity, and architecture width and depth.

\begin{remark} \normalfont
While the filtering operation in \eqref{eq:randomGraphFilter} resembles the form of the edge-variant graph filter \cite{Coutino2019}, these two are substantially different concepts. The edge-variant graph filter weighs differently the neighborhood information with a set of edge-weight matrices $\{ \bbA_k \}_{k=1}^K$, which share the same graph topology but assign different weights to graph edges \cite{Coutino2019}. These weights are trainable or can be designed for the task at hand. However, the filtering operation in \eqref{eq:randomGraphFilter} runs over a sequence of non-trainable (or non-designable) graph shift operators $\{ \bbS_k \}_{k=1}^K$ because of stochastic graph perturbations, which are imposed by external factors (e.g., environment influence, human effects, or adversarial attacks) but are not design choices. These perturbed shift operators $\{ \bbS_k \}_{k=1}^K$ capture the random graph topologies determined by encountered stochastic perturbations, and cannot be learned during training. The trainable parameters in \eqref{eq:randomGraphFilter} are only the filter coefficients $\{ h_k \}_{k=0}^K$.
\end{remark}

\section{Stability of Graph Filters}\label{sec:GraphStability}

As it follows from \eqref{eq:graphFilter}-\eqref{eq:GNNArchi}, analyzing the robustness of the GCNN to stochastic perturbations requires analyzing first the robustness of the graph filter. For this analysis, we consider edges drop randomly with a given probability as defined by the following random edge sampling (RES) model \cite{Isufi17}.

\begin{definition}[Random Edge Sampling]
Given the underlying graph $\ccalG = \{ \ccalV, \ccalE \}$ and an edge sampling probability $p$, the RES($\ccalG, p$) model generates random graphs $\ccalG_k = \{ \ccalV, \ccalE_k \}$ with the same node set $\ccalV$ and a random edge set $\ccalE_k \in \ccalE$ such that edge $(i,j)$ exists in $\ccalG_k$ with probability $p$, i.e.,
\begin{equation}\label{eq:RESmodel}
\text{Pr} \left[(i,j) \in \ccalE_k\right] = p,~\forall~ (i,j) \in \ccalE
\end{equation}
where $\text{Pr} \left[\cdot\right]$ is the probability measure.
\end{definition}

\noindent The RES($\ccalG, p$) model implies the subgraph realization $\ccalG_k$ has edges $(i,j)\in \ccalE_k$ drawn independently from the underlying edge set $\ccalE$ with probability $p$, leading to a different shift operator $\bbS_k$. This model is motivated by practical distributed applications over physical networks, where the graph is determined by the physicalities of the network and the graph edges correspond to the physical links. In these scenarios, edges (i.e., links) drop randomly due to independent external factors and this results in stochastic graph perturbations that adhere to the RES model. Therefore, running a filter over a sequence of RES($\ccalG,p$) graphs $\ccalG_1, \ldots, \ccalG_K$ leads to a random filter output [cf. \eqref{eq:randomGraphFilter}] that differs from the filter output on the nominal graph $\ccalG$ [cf. \eqref{eq:graphFilter}]. The deviation depends on how the signal is propagated through the chain of random shift operators $\bbS_1, \ldots, \bbS_K$.

To characterize the effects of stochastic perturbations on the filter output, we conduct the analysis in the graph spectral domain. The spectral domain has also been used to characterize the stability of the GCNN to small deterministic graph perturbations \cite{gama2020stability}. However, stochastic perturbations are not relatively small making the earlier deterministic results \cite{gama2020stability} inapplicable in the stochastic setting. To overcome this issue, we propose the novel concept of generalized graph filter frequency response allowing for the spectral analysis of graph filters over random graphs. Subsequently, we define the Lipschitz gradient and generalize the integral Lipschitz condition for graph filters from the deterministic setting to the stochastic setting. These tools allow characterizing the GCNN stability mathematically and provide insights on the role played by the graph stochasticity as well as different actors.

\subsection{Filter Frequency Response over Random Graphs}\label{subsec:generalizedFrequency}

\begin{figure}[t]
\centering

\def \thisplotscale {3.15}
\def \unit {\thisplotscale cm}

\def \frequencyresponse 
     {   0.7*exp(-(1*(x-1.2))^2) 
       + 0.8*exp(-(0.7*(x-4))^2) 
       + 0.6*exp(-(1.4*(x-6))^2) 
       + 0.1}

{\footnotesize
\begin{tikzpicture}[x = 1.0*\unit, y=0.8*\unit]

\def \factorx {2.6/8}
\def \deltax  {0.5*\factorx}
\def \shadeshift  {0.05}

\path [fill=black!20, opacity = 0.5] 
              (\deltax - 0.001*\factorx - \shadeshift, 0.00) rectangle 
              (\deltax + 0.030*\factorx + \shadeshift, 1.00);
\path [fill=black!20, opacity = 0.5] 
              (\deltax + 1.413*\factorx - \shadeshift, 0.00) rectangle 
              (\deltax + 1.67*\factorx + \shadeshift, 1.00);
\path [fill=black!20, opacity = 0.5] 
              (\deltax + 3.393*\factorx - \shadeshift, 0.00) rectangle 
              (\deltax + 3.770*\factorx + \shadeshift, 1.00);
\path [fill=black!20, opacity = 0.5] 
              (\deltax + 6.048*\factorx - \shadeshift, 0.00) rectangle 
              (\deltax + 6.720*\factorx + \shadeshift, 1.00);

\begin{axis}[scale only axis,
             width  = 2.6*\unit,
             height = 0.8*\unit,
             xmin = -0.5, xmax=7.5,
             xtick = {0.03, -0.01, 1.67, 1.413, 3.77, 3.393, 6.72, 6.048},
             xticklabels = {\red{$\qquad\lambda_1\phantom{\lambda}$},
                            \blue{$\lambda_1\ \ $}, 
                            \red{$\quad\lambda_i\phantom{\lambda}$}, 
                            \blue{$\lambda_i$},
                            \red{$\quad\lambda_j\phantom{\lambda}$}, 
                            \blue{$\lambda_j$},
                            \red{$\quad\lambda_{n}\phantom{\lambda}$},
                            \blue{$\lambda_n$}},
             ymin = -0, ymax = 1.15,
             ytick = {1.15},
             yticklabels = {},
             enlarge x limits=false]

\addplot+[samples at = {0.03, 1.67, 
                        3.77, 6.72}, dashed, 
          color = red!60, 
          ycomb, 
          mark=otimes*, 
          mark options={red!60}]
         {\frequencyresponse};

\addplot+[samples at = {-0.01, 1.413, 
                        3.393, 6.048}, dashed, 
          color = blue!60, 
          ycomb, 
          mark=oplus*, 
          mark options={blue!60}]
         {\frequencyresponse};

\addplot[ domain=-0.5:7.5, 
          samples = 80, 
          color = black,
          line width = 1.2]
         {\frequencyresponse};

\end{axis}
\end{tikzpicture}}

\caption{The frequency response $h(\lambda)$ of a graph filter (black line). The function $h(\lambda)$ is independent of the graph. A specific graph instantiates $h(\lambda)$ on specific eigenvalues of the shift operator. We highlight the latter for two different graphs in red and blue.} 
\label{fig:frequencyResponse}
\end{figure}
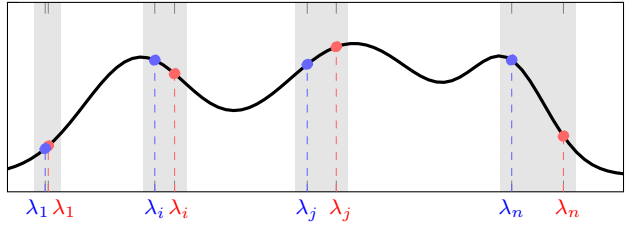

Since the shift operator $\bbS$ is symmetric, it accepts the eigendecomposition $\bbS = \bbV \bbLambda \bbV^\top$ with orthogonal eigenvectors $\bbV = [\bbv_1,\ldots,\bbv_n]$ and eigenvalues $\bbLambda = \text{diag} (\lambda_1,\ldots,\lambda_n)$. The graph Fourier transform (GFT) projects the signal $\bbx$ on the eigenvector basis $\bbV$ as $\bbx = \sum_{i=1}^n \hat{x}_i \bbv_i$, where $\hat{\bbx} = [\hat{x}_1,\ldots,\hat{x}_n]^\top$ collects the graph Fourier coefficients of $\bbx$ \cite{ortega2018}. Substituting this GFT into the graph filter [cf. \eqref{eq:graphFilter}], we have
\begin{equation}\label{eq:FilterGFT}
\bbu = \sum_{k=0}^K h_k \bbS^k \sum_{i=1}^n \hat{x}_i \bbv_i = \sum_{i=1}^n \hat{x}_i \sum_{k=0}^K h_k \lambda_i^k \bbv_i.
\end{equation}
By considering also the GFT of the output $\hat{\bbu} = \sum_{i=1}^n \hat{u}_i \bbv_i$, we can write \eqref{eq:FilterGFT} as $\hat{\bbu} = \bbH(\bbLambda)\hat{\bbx}$ where the diagonal matrix $\bbH(\bbLambda) = \sum_{k=0}^K h_k \bbLambda^k$ contains the filter frequency response on the main diagonal with $i$th diagonal entry $h(\lambda_i) = \sum_{k=0}^K h_k \lambda_i^k$. The latter is the graph filter frequency response evaluated at the eigenvalue $\lambda_i$. Note though that for a defined set of coefficients $\{h_k\}_{k=1}^K$, the filter frequency response is an analytic function of the form
\begin{equation}\label{eq:FrequencyResponse}
h(\lambda) = \sum_{k=0}^K h_k \lambda^k
\end{equation}
for a generic graph frequency variable $\lambda$. That is, it is a universal shape and independent on the underlying graph. A specific graph $(\ccalG, \bbS)$ only instantiates the eigenvalues $\{ \lambda_i \}_{i=1}^n$ on the function variable $\lambda$ in \eqref{eq:FrequencyResponse}, which results in specific filter responses $\{ h(\lambda_i) \}_{i=1}^n$---see Fig. \ref{fig:frequencyResponse}. Thus, if the filter coefficients are learned over a graph $\ccalG_1$ (blue eigenvalues in Fig. \ref{fig:frequencyResponse}) but implemented over another graph $\ccalG_2$ (red eigenvalues in Fig. \ref{fig:frequencyResponse}) that deviates substantially from $\ccalG_1$, the filter will implement an arbitrary different frequency response. One way to mitigate this effect is to work with graph filters that are integral Lipschitz as we recall next.

\begin{figure}[t]
\centering

\def \thisplotscale {3.15}
\def \unit {\thisplotscale cm}

\def \frequencyresponseC 
{ 1.35*exp(-(9.0*(x-0.0))^2) }
\def \frequencyresponseTwoC
{ 1.35*exp(-(6.0*(x-0.4))^2) }
\def \frequencyresponseTwoCV
{ 1.35*exp(-(4.5*(x-0.9))^2) }
\def \frequencyresponseThreeC
{ 1.35*exp(-(2.*(x-1.8))^2) }
\def \frequencyresponseFourC
{ 1.35*exp(-(1.5*(x-3.2))^2) }
\def \frequencyresponseFiveC
{ 1.35 - 1.35*exp(-(1.1*(x-3.2))^2) }

\def \frequencyresponse 
{ 1.15*exp(-(6.0*(x-0.0))^2) }
\def \frequencyresponseTwo
{ 1.15*exp(-(3.0*(x-0.6))^2) }
\def \frequencyresponseThree
{ 1.15*exp(-(1.5*(x-1.8))^2) }
\def \frequencyresponseFour
{ 1.15*exp(-(1.0*(x-3.6))^2) }
\def \frequencyresponseFive
{ 1.15 - 1.15*exp(-(0.7*(x-3.0))^2) }

\begin{tikzpicture}[x = 1*\unit, y=0.8*\unit]

\def \factorx {2.6/8}
\def \deltax  {0.5*\factorx}
\def \shadeshift  {0.05}

\path [fill=black!20, opacity = 0.5] 
              (\deltax - 0.001*\factorx - \shadeshift, 0.00) rectangle 
              (\deltax + 0.030*\factorx + \shadeshift, 1.00);
\path [fill=black!20, opacity = 0.5] 
              (\deltax + 1.413*\factorx - \shadeshift, 0.00) rectangle 
              (\deltax + 1.67*\factorx + \shadeshift, 1.00);
\path [fill=black!20, opacity = 0.5] 
              (\deltax + 3.393*\factorx - \shadeshift, 0.00) rectangle 
              (\deltax + 3.770*\factorx + \shadeshift, 1.00);
\path [fill=black!20, opacity = 0.5] 
              (\deltax + 6.048*\factorx - \shadeshift, 0.00) rectangle 
              (\deltax + 6.720*\factorx + \shadeshift, 1.00);

\begin{axis}[scale only axis,
             width  = 2.6*\unit,
             height = 0.8*\unit,
             xmin = -0.5, xmax=7.5,
             xtick = {0.03, -0.01, 1.67, 1.413, 3.77, 3.393, 6.72, 6.048},
             xticklabels = {\red{$\qquad\lambda_1\phantom{\lambda}$},
                            \blue{$\lambda_1\ \ $}, 
                            \red{$\quad\lambda_i\phantom{\lambda}$}, 
                            \blue{$\lambda_i$},
                            \red{$\quad\lambda_j\phantom{\lambda}$}, 
                            \blue{$\lambda_j$},
                            \red{$\quad\lambda_{n}\phantom{\lambda}$},
                            \blue{$\lambda_n$}},
             ymin = -0, ymax = 1.4,
             ytick = {1.4},
             yticklabels = {}]

\addplot[domain=-0.5:0.6,samples = 120, color = black, line width=1.2] {\frequencyresponse};
\addplot[domain= 0.3:3.2,samples = 80, color = black, line width=1.2] {\frequencyresponseThree};
\addplot[domain= 3.0:7.5,samples = 80, color = black, line width=1.2] {\frequencyresponseFive};

\addplot+[samples at = {0.03}, dashed, 
          color = red!60, 
          ycomb, 
          mark=otimes*, 
          mark options={red!60}]
         {\frequencyresponse};

\addplot+[samples at = {-0.01}, dashed,
          color = blue!60, 
          ycomb, 
          mark=oplus*, 
          mark options={blue!60}]
         {\frequencyresponse};

\addplot+[samples at = {1.67}, dashed, 
          color = red!60, 
          ycomb, 
          mark=otimes*, 
          mark options={red!60}]
         {\frequencyresponseThree};

\addplot+[samples at = {1.413}, dashed, 
          color = blue!60, 
          ycomb, 
          mark=oplus*, 
          mark options={blue!60}]
         {\frequencyresponseThree};
         
\addplot+[samples at = {3.77}, dashed, 
          color = red!60, 
          ycomb, 
          mark=otimes*, 
          mark options={red!60}]
         {\frequencyresponseFive};

\addplot+[samples at = {3.393}, dashed, 
          color = blue!60, 
          ycomb, 
          mark=oplus*, 
          mark options={blue!60}]
         {\frequencyresponseFive};
         
\addplot+[samples at = {6.72}, dashed, 
          color = red!60, 
          ycomb, 
          mark=otimes*, 
          mark options={red!60}]
         {\frequencyresponseFive};

\addplot+[samples at = {6.048}, dashed, 
          color = blue!60, 
          ycomb, 
          mark=oplus*, 
          mark options={blue!60}]
         {\frequencyresponseFive};

\end{axis}
\end{tikzpicture}
\caption{The frequency response $h(\lambda)$ of a integral Lipschitz filter (black line). The change of eigenvalue $\lambda$ leads to a bounded change of frequency response $h(\lambda)$. This behavior is emphasized for large-value frequencies, where $h(\lambda)$ is nearly constant.} 
\label{fig:IntegralfrequencyResponse}
\end{figure}
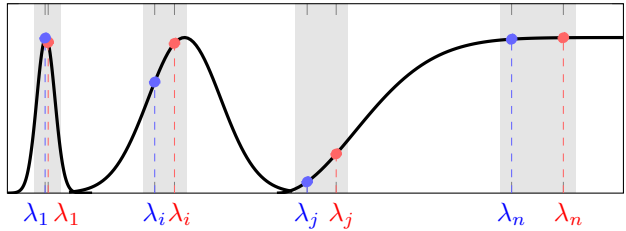

\begin{definition}[Integral Lipschitz filter]
Consider a graph filter with the frequency response $h(\lambda)$ [cf. \eqref{eq:FrequencyResponse}] satisfying $|h(\lambda)|\le 1$. The filter is integral Lipschitz if for any frequency $\lambda$ in a finite space $\Lambda$, there exists a constant $C_L>0$ such that
\begin{equation}\label{eq:IntegralFrequencyResponse}
|h'(\lambda)| \le C_L~~\text{and}~~ |\lambda h'(\lambda)| \le C_L 
\end{equation}
where $h'(\lambda)$ is the derivative of $h(\lambda)$.
\end{definition}
\noindent An integral Lipschitz filter has basically a restricted variability in its frequency response, i.e., function $h(\lambda)$ does not change faster than linear and tends to flatten for large $\lambda$---see Fig. \ref{fig:IntegralfrequencyResponse}. This flat response in the high-eigenvalue region --where the deviation in the topology is more emphasized in the spectrum-- makes the filter more robust to topological perturbations. However, this robustness comes at expenses of the discriminatory power. Integral Lipschitz filters cannot discriminate between spectral features in large-value frequencies \cite{gama2020stability}. Constant $C_L$ controls the variability of $h(\lambda)$ to change the trade-off between robustness and discriminability.

The integral Lipschitz property has been shown useful to characterize the filter robustness when run over a deterministic perturbed shift operator $\bbS+\Delta \bbS$ but not over a sequence of random shift operators [cf. \eqref{eq:randomGraphFilter}]. To characterize the filter robustness in the latter case, we need to generalize the integral Lipschitz property to graph filters over a sequence of random graphs. For this, we first generalize the frequency response of a filter over a sequence of random graphs. 

\textbf{Generalized frequency response.} Let $\bbS_1,\ldots,\bbS_K$ be the shift operators of the RES($\ccalG,p$) graphs  $\ccalG_1,\ldots,\ccalG_K$, respectively. Let also $\bbS_k = \bbV_k \bbLambda_k \bbV_k^\top$ be the eigendecomposition of the $k$th shift operator with eigenvectors $\bbV_k = [\bbv_{k1}, \ldots, \bbv_{kn}]$ and eigenvalues $\bbLambda_k = \text{diag}(\lambda_{k1}, \ldots, \lambda_{kn})$. As it follows from \eqref{eq:randomGraphFilter}, computing the filter output $\tilde{\bbu}$ requires computing the shifted signals $\bbx^{(k)} = \bbS_k\cdots \bbS_1\bbx$. From the eigendecomposition $\bbS_1 = \bbV_1 \bbLambda_1 \bbV_1^\top$, we can write the signal as $\bbx = \sum_{i_1=1}^n \hat{x}_{1i_1} \bbv_{1i_1}$ where $\hat{\bbx}_1 = [\hat{x}_{11}, \ldots, \hat{x}_{1n}]^\top$ are the graph Fourier coefficients of $\bbx$ over $\bbS_1$. Then, the one-shifted signal $\bbx^{(1)} = \bbS_1 \bbx$ can be written as
\begin{equation}\label{eq:GeneralizedFrequencyResponse1}
\bbx^{(1)} = \bbS_1 \bbx = \bbS_1 \sum_{i_1=1}^n \hat{x}_{1i_1} \bbv_{1i_1} = \sum_{i_1=1}^n \hat{x}_{1i_1} \lambda_{1i_1} \bbv_{1i_1}.
\end{equation}
For the two-shifted signal $\bbx^{(2)} = \bbS_2 \bbS_1 \bbx$, we treat each eigenvector $\bbv_{1i_1}$ in \eqref{eq:GeneralizedFrequencyResponse1} as an intermediate graph signal and further decompose it with the GFT over $\bbS_2$. That is, we write $\bbv_{1 i_1} = \sum_{i_2=1}^n \hat{x}_{2 i_1 i_2} \bbv_{2 i_2}$ with $\hat{\bbx}_{2i_1} = [\hat{x}_{2i_11}, \ldots, \hat{x}_{2i_1n}]^\top$ being the graph Fourier coefficients of $\bbv_{1 i_1}$ over $\bbS_2$. Using these GFTs for all the eigenvectors $\bbv_{11},\ldots,\bbv_{1n}$ and \eqref{eq:GeneralizedFrequencyResponse1}, we can represent the two-shifted signal $\bbx^{(2)}$ as
\begin{equation}\label{eq:GeneralizedFrequencyResponse2}
\bbx^{(2)} \!=\! \bbS_2 \sum_{i_1=1}^n \hat{x}_{1i_1} \lambda_{1i_1} \bbv_{1i_1} \!=\! \sum_{i_2=1}^n \sum_{i_1=1}^n \hat{x}_{2i_1i_2}\hat{x}_{1i_1}\lambda_{2 i_2} \lambda_{1i_1} \bbv_{1i_1}.
\end{equation}
The coefficients $\{ \hat{x}_{1i_1} \}_{i_1=1}^n$ and $\{ \hat{x}_{2i_1i_2} \}_{i_1i_2}$ characterize the generalized GFT of $\bbx$ over a chain of two shift operators $\bbS_2 \bbS_1$. Proceeding in a recursive way, we can generalize the GFT to the $k$-shifted signal $\bbx^{(k)} = \bbS_k \cdots \bbS_1 \bbx$ as
\begin{equation}\label{eq:GeneralizedFrequencyResponsek}
\bbx^{(k)} \!=\! \sum_{i_k=1}^n \cdots \sum_{i_1=1}^n \hat{x}_{ki_{k-1}i_k}\cdots  \hat{x}_{2i_1i_2}\hat{x}_{1i_1} \prod_{j=1}^k \lambda_{ji_j} \bbv_{ki_k}
\end{equation}
with coefficients $\{ \hat{x}_{1i_1} \}_{i_1=1}^n$ and $\{ \hat{x}_{2i_ji_{j+1}} \}_{j=1}^{k-1}$. Note that representation \eqref{eq:GeneralizedFrequencyResponsek} depends on all eigenvalues $\bbLambda_k, ..., \bbLambda_1$ and eigenvectors matrices $\bbV_k, \ldots, \bbV_1$ of the $k$ successive shift operators $\bbS_k, \ldots, \bbS_1$. Aggregating the $K+1$ shifted signals $\{ \bbx^{(k)} \}_{k=0}^K$ [cf. \eqref{eq:GeneralizedFrequencyResponsek}], the generalized GFT of the filter output over $K$ random graphs [cf. \eqref{eq:randomGraphFilter}] becomes
\begin{equation}\label{eq:GeneralizedFrequencyResponseFilter}
\hat{\tilde{\bbu}} \!=\! \sum_{i_K=1}^n \!\cdots\! \sum_{i_1=1}^n \hat{x}_{Ki_{K-1}i_K}\cdots  \hat{x}_{2i_1i_2}\hat{x}_{1i_1} \sum_{k=1}^K h_k \prod_{j=0}^k \lambda_{ji_j} \bbv_{Ki_k}
\end{equation}
where $\{ \hat{x}_{1i_1} \}_{i_1=1}^n$ and $\{ \hat{x}_{2i_ji_{j+1}} \}_{j=1}^{K-1}$ are the generalized GFT of $\bbx$ over a chain of $K$ shift operators $\bbS_K \cdots \bbS_1$ and $\lambda_{0i_0}=1$ since $\bbS_0 = \bbI$. With this in place, we formalize next the generalized frequency response of a filter over a sequence of random graphs.

\begin{definition}[Generalized graph filter frequency response]\label{def:IntegralLIpschitz}
Consider a graph filter defined by coefficients $\{ h_k \}_{k=0}^K$ run over a sequence of $K$ shift operators [cf. \eqref{eq:randomGraphFilter}]. The generalized frequency response of the filter is the $K$-dimensional analytic function
\begin{equation}\label{eq:GeneralizedFrequencyResponse}
h(\bblambda) = \sum_{k=0}^K h_k \prod_{\kappa=0}^k \lambda_{\kappa}
\end{equation}
for a generic vector variable $\bblambda = [\lambda_{1},\ldots,\lambda_{K}]^\top \in \mathbb{R}^K$, where $\lambda_0=1$ by default.
\end{definition}
\noindent The generalized frequency response $h(\bblambda)$ is a multivariate function, in which the $k$th entry $\lambda_{k}$ is the analytic frequency variable corresponding to the shift operator $\bbS_k$\footnote{This is in analogy to \eqref{eq:FrequencyResponse} for the single shift operator case. We will use subscript $k$ to indicate that $\lambda_k$ is an analytic frequency variable associated to the shift operartor $\bbS_k$ [cf. \eqref{eq:GeneralizedFrequencyResponse}] and subscript $i$ to indicate that $\lambda_{ki}$ is the specific $i$th eigenvalue of a specific shift operaror $\bbS_k$ [cf. \eqref{eq:FilterGFT}]. For clarity, we will also associate to $\lambda_k$ or $\lambda_{ki}$ the terminologies \emph{analytic} and \emph{specific}, respectively.}. The multidimensional shape of the generalized frequency response $h(\bblambda)$ is determined entirely by the filter coefficients $\{ h_k \}_{k=0}^K$, while a specific chain of shift operators $\{ \bbS_k\}_{k=1}^K$ only specifies the $K$-dimensional frequency vector variable $\bblambda$---see Fig. \ref{fig:GeneralizedFrequencyResponse} for $K=2$. From this perspective, results built upon the generalized frequency response hold uniformly for any specific graph. The following remark concludes this section.

\begin{remark}[Deterministic frequency response]
When the link sampling probability is $p = 1$, all RES($\ccalG,p$) subgraph realizations are identical to the underlying graph $\ccalG$. Thus, we have $\bbS_1 = \cdots = \bbS_K = \bbS$ and the generalized GFT in \eqref{eq:GeneralizedFrequencyResponseFilter} reduces to the GFT in \eqref{eq:FilterGFT}; the multivariate generalized frequency response $h(\bblambda)$ in \eqref{eq:GeneralizedFrequencyResponse}] reduces to the univariate frequency response $h(\lambda)$ in \eqref{eq:FrequencyResponse}.
\end{remark}

\begin{figure}[t]
\centering
\includegraphics[width=0.85\linewidth , height=0.55\linewidth, trim=10 10 10 10]{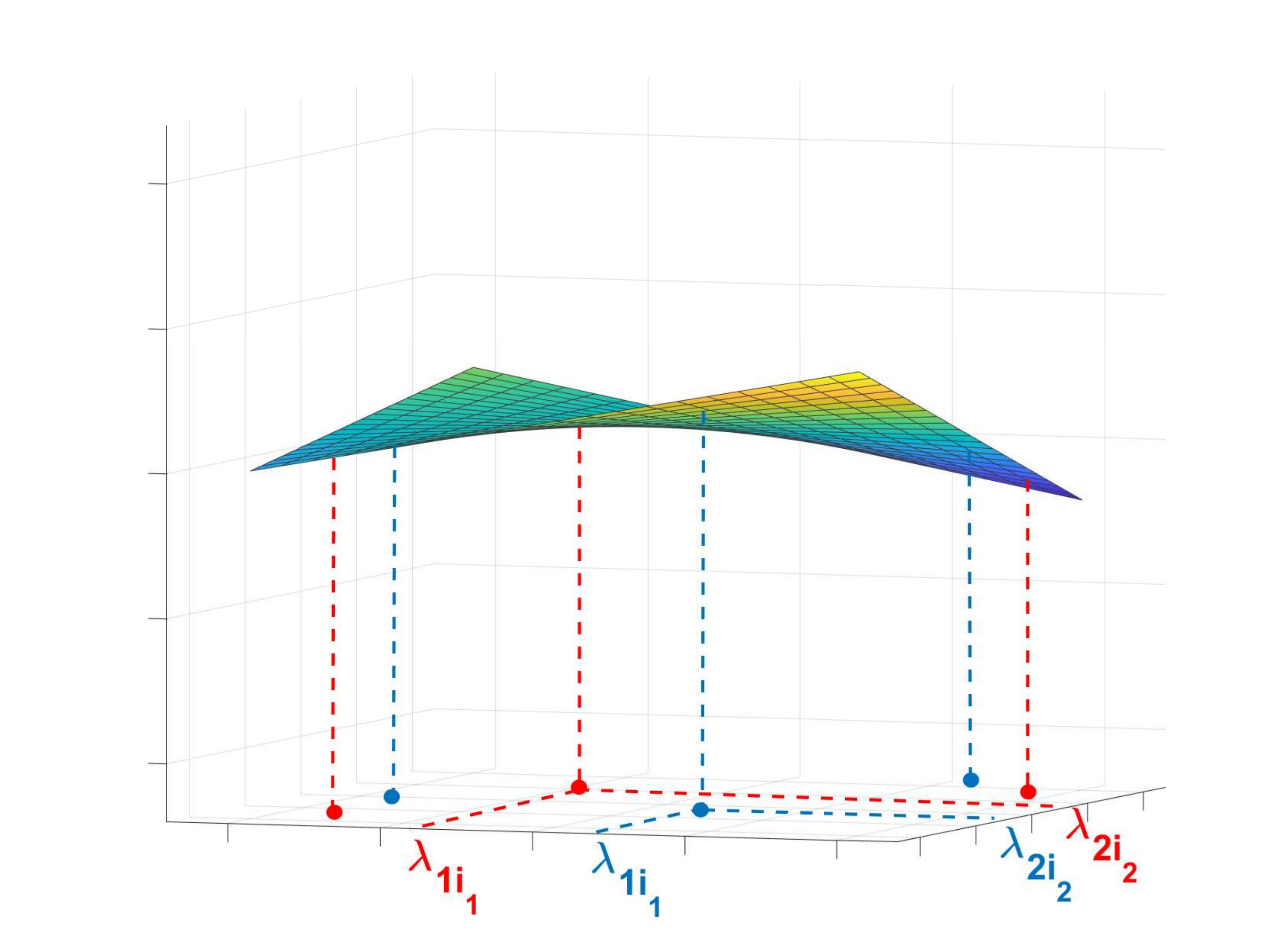}
\caption{The two-dimensional generalized frequency response of a graph filter (black line). The function $h(\bblambda)$ is independent of random graph realizations. For two specific chains of random graph realizations $\bbS_2\bbS_1$ (red) and $\bbS_2^\prime\bbS_1^\prime$ (blue), $h(\bblambda)$ is instantiated on specific multivariate frequencies. }
\label{fig:GeneralizedFrequencyResponse}
\end{figure}

\subsection{Stochastic Perturbations on Graph Filters}\label{subsec:FilterStability}

We now define the generalized integral Lipschitz property for the generalized frequency response, likewise in Def. \ref{def:IntegralLIpschitz}. For this, we first define the Lipschitz gradient.

\begin{definition}[Lipschitz gradient]\label{def:LipschitzGradient}
Consider the analytic generalized frequency response $h(\bblambda)$ with coefficients $\{ h_k \}_{k=0}^K$ [cf. \eqref{eq:GeneralizedFrequencyResponse}]. Consider also two specific instantiations $\bblambda_1 = [\lambda_{11}, \ldots, \lambda_{1K}]^\top$ and $\bblambda_2 = [\lambda_{21}, \ldots, \lambda_{2K}]^\top$ of the analytic multivariate frequency vector variable $\bblambda$. Further, let $\bblambda^{(k)} = [\lambda_{11}, \ldots, \lambda_{1k}, \lambda_{2(k+1)}, \ldots, \lambda_{2K}]^\top$ be the vector formed by concatenating the first $k$ entries of $\bblambda_1$ and the last $K-k$ entries of $\bblambda_2$. The Lipschitz gradient of $h(\bblambda)$ between $\bblambda_1$ and $\bblambda_2$ is
\begin{equation}\label{eq:LipschitzGradient}
\nabla_L h(\bblambda_1, \bblambda_2) = \Big[\frac{\partial h(\bblambda^{(1)})}{\partial \lambda_1}, \ldots, \frac{\partial h(\bblambda^{(K)})}{\partial \lambda_K}\Big]^\top
\end{equation}
where $\partial h(\bblambda^{(k)})/\partial \lambda_k$ is the partial derivative w.r.t. the $k$-th entry $\lambda_k$ at $\bblambda^{(k)}$.
\end{definition}

\noindent The Lipschitz gradient $\nabla_L h(\bblambda_1, \bblambda_2)$ characterizes the variability of the analytic function $h(\bblambda)$ when specified at two multivariate frequencies $\bblambda_1$ and $\bblambda_2$. We illustrate the latter more explicitly with the following lemma.

\begin{lemma}\label{lemma:LipschitzGradient}
Consider the analytic generalized frequency response $h(\bblambda)$ with coefficients $\{ h_k \}_{k=0}^K$ [cf. (13)]. Let $\bblambda_1 = [\lambda_{11}, \ldots, \lambda_{1K}]^\top$ and $\bblambda_2 = [\lambda_{21}, \ldots, \lambda_{2K}]^\top$ be two multivariate frequency vectors, and $\nabla_L h(\bblambda_1, \bblambda_2)$ be the Lipschitz gradient of $h(\bblambda)$ between $\bblambda_1$ and $\bblambda_2$ [cf. Def. 1]. Then, it holds that
\begin{equation}\label{eq:LipschitzGradientCondition}
h(\bblambda_1) - h(\bblambda_2) = \nabla_L^\top h(\bblambda_1, \bblambda_2) (\bblambda_1 - \bblambda_2).
\end{equation}
\end{lemma}

\begin{proof}
See \ref{proof:lemmaLischitz}.
\end{proof}

With above preliminaries, we define the generalized integral Lipschitz filter.

\begin{definition}[Generalized integral Lipschitz filter]\label{def:GeneralizedIntegralLIpschitz}
Consider a graph filter run over random graphs with the analytic generalized frequency response $h(\bblambda)$ in \eqref{eq:GeneralizedFrequencyResponse} satisfying $|h(\bblambda)| \le 1$. The filter is generalized integral Lipschitz if for any specific $K$-dimensional frequency vectors $\bblambda_1$, $\bblambda_2$ in a finite space $\Lambda^K$, there exists a constant $C_L >0$ such that
\begin{equation}\label{eq:GeneralizedIntegralLipschitzFilter}
\| \nabla_L h(\bblambda_1, \bblambda_2) \|_2 \!\le\! C_L~~\text{and}~~\| \bblambda_1 \!\odot\! \nabla_L h(\bblambda_1, \bblambda_2) \|_2 \!\le\! C_L
\end{equation}
where $\nabla_L h(\bblambda_1, \bblambda_2)$ is the Lipschitz gradient [Def. \ref{def:LipschitzGradient}] and $\odot$ is the elementwise product of vectors.
\end{definition}
\noindent In essence, Def. \ref{def:GeneralizedIntegralLIpschitz} implies that the generalized integral Lipschitz filter $h(\bblambda)$ does not change faster than linear in the $K$-dimensional space. That is, analogous to the Lipschitz continuity of the one-dimensional function $f(x)$, i.e., $|f(y) - f(x)|/|y-x| \le C_L~\forall~x, y \in \mathbb{R}$, Def. \ref{def:GeneralizedIntegralLIpschitz} yields 
\begin{align}\label{eq:LipschitzIndication}
|h(\bblambda_1) - h(\bblambda_2)| &= |\nabla_L^\top h(\bblambda_1, \bblambda_2) (\bblambda_1 - \bblambda_2)| \\
&\le \|\nabla_L h(\bblambda_1, \bblambda_2)\|_2 \|\bblambda_1 - \bblambda_2\|_2 \nonumber \le C_L \|\bblambda_1 - \bblambda_2\|_2
\end{align}
for all $\bblambda_1, \bblambda_2 \in \mathbb{R}^K$, where Lemma \ref{lemma:LipschitzGradient} is used in the first equality, the triangular inequality is used in the second inequality, and the Lipschitz condition [cf. \eqref{eq:GeneralizedIntegralLipschitzFilter}] is used in the last inequality. Def. \ref{def:GeneralizedIntegralLIpschitz} says also that the generalized frequency response $h(\bblambda)$ changes more slowly when the analytic variable $\bblambda$ is specified at large values; i.e., different instantiations $\bblambda_1$ and $\bblambda_2$ may yield similar frequency responses $h(\bblambda_1)$ and $h(\bblambda_2)$. Note also that Def. \ref{def:GeneralizedIntegralLIpschitz} generalizes Def. \ref{def:IntegralLIpschitz} to the multivariate case\footnote{Particularizing $\bbS_1=\ldots = \bbS_K = \bbS$ to the deterministic underlying shift operator, condition \eqref{eq:GeneralizedIntegralLipschitzFilter} reduces to condition \eqref{eq:IntegralFrequencyResponse} and the generalized integral Lipschitz filter Def. \ref{def:GeneralizedIntegralLIpschitz} reduces to the integral Lipschitz filter Def. \ref{def:IntegralLIpschitz}.}.

Upon defining the generalized integral Lipschitz filter, we can quantify the stability of the graph filter to stochastic perturbations.

\begin{theorem}\label{theorem:filterStability}
Consider the graph filter $\bbH(\bbS)$ [cf. \eqref{eq:graphFilter}] with underlying shift operator $\bbS$ and filter coefficients $\{ h_k \}_{k=0}^K$. Consider also the filter $\tilde{\bbH}(\bbS)$ run over RES($\ccalG, p$) subgraph realizations $\bbS_k$ for $k=1,\ldots,K$ [cf. \eqref{eq:randomGraphFilter}]. Let the filter be generalized integral Lipschitz with constant $C_L$ [Def. \ref{def:GeneralizedIntegralLIpschitz}]. Then, for any graph signal $\bbx$, the expected difference between the filter output run over random graphs $\tilde{\bbH}(\bbS)\bbx$ and that run over the nominal graph $\bbH(\bbS)\bbx$ is upper bounded as
 \begin{equation}\label{eq:FilterStability}
\mathbb{E}\left[\| \tilde{\bbH}(\bbS)\bbx \!-\! \bbH(\bbS)\bbx \|^2_2\right] \!\le\! C (1\!-\!p) \| \bbx \|^2_2 \!+\! \ccalO\big((1\!-\!p)^2\big)
\end{equation}
where $C = n \alpha C_L^2$ is the stability constant and scalar $\alpha$ is either the maximal node degree or $2$ depending on the shift operator [cf. Lemma \ref{lemma:traceOperation}].
\end{theorem}
\begin{proof}
See \ref{proof:theorem1}.
\end{proof}

Theorem \ref{theorem:filterStability} states that graph filters are Lipschitz stable to stochastic perturbations. That is, the expected difference of the filter output when run over the nominal graph $\ccalG$ and over RES($\ccalG,p$) realizations $\{\ccalG_k\}_{k=1}^K$ is upper bounded linearly by the link loss probability $1-p$ w.r.t. a stability constant $C$. When the RES($\ccalG,p$) realizations are almost stable (i.e., $p \to 1$), the bound approaches zero and the graph filter maintains its performance. The stability constant $C$ embeds the role of the filter and the graph: $C_L$ is the generalized Lipschitz constant [cf. \eqref{eq:IntegralFrequencyResponse}] and $\alpha$ depends on the graph shift operator. For instance $\alpha$ is the maximum degree of $\ccalG$ if $\bbS$ is the adjacency matrix and $\alpha = 2$ if $\bbS$ is the Laplacian matrix \cite{Zhan2020}. Therefore, we can affect the filter robustness through constant $C$ by designing the filter coefficients $\{ h_k \}_{k=0}^K$ to lower the generalized Lipschitz constant $C_L$ in \eqref{eq:GeneralizedIntegralLipschitzFilter}.

\section{Stability of Graph Convolutional Neural Networks}\label{sec:GCNNStability}

The stability analysis of GCNNs to stochastic perturbations is more challenging than that of graph filters because of the nonlinearities. One typical approach to handle the latter is to consider nonlinearities that are Lipschitz as stated by the following definition \cite{kaszkurewicz1995comments}.

\begin{definition}[Lipschitz nonlinearity]\label{def:LipschitzNonlinearity}
The nonlinearity $\sigma(\cdot)$ satisfying $\sigma(0)=0$ is Lipschitz if there exists a constant $C_\sigma > 0$ such that for any $a, b \in \mathbb{R}$, it holds that
 \begin{equation}\label{eq:LipschitzNonlinear}
|\sigma(a) - \sigma(b)| \le C_\sigma|a - b|.
\end{equation}
\end{definition}
\noindent The typical nonlinearities used in the GCNN (e.g., ReLU, absolute value, hyperbolic tangent) satisfy Def. \ref{def:LipschitzNonlinearity}. 

The following theorem formally states the stability of the GCNN to stochastic perturbations.

\begin{theorem} \label{theorem:GNNstability}
Consider the GCNN $\bbPhi(\bbx; \bbS, \ccalH)$ of $L$ layers and $F$ features per layer with underlying shift operator $\bbS$ and parameters $\ccalH$ [cf. \eqref{eq:GNNArchi}]. Consider also the same GCNN $\tilde{\bbPhi}(\bbx;\bbS,\ccalH)$ run over RES($\ccalG,p$) subgraph realizations $\{ \bbS_{k\ell}^{fg} \}$. Furthermore, let also the filters be generalized integral Lipschitz with constant $C_L$ [Def. \ref{def:GeneralizedIntegralLIpschitz}] and the nonlinearity $\sigma(\cdot)$ be Lipschitz with constant $C_\sigma$ [Def. \ref{def:LipschitzNonlinearity}]. Then, for any graph signal $\bbx$, the expected difference between the GCNN output run over random graphs $\tilde{\bbPhi}(\bbx; \bbS, \ccalH)$ and that run over the nominal graph $\bbPhi(\bbx; \bbS, \ccalH)$ is upper bounded as
\begin{equation}\label{eq:GNNstability}
\begin{split}
&\mathbb{E} \left[ \| \tilde{\bbPhi}(\bbx;\bbS,\ccalH) - \bbPhi(\bbx;\bbS,\ccalH)\|^2_2 \right]\le C (1-p) \| \bbx \|_2^2 \!+\! \mathcal{O}((1-p)^2)
\end{split}
\end{equation}
where $C= n \alpha C_L^2 L^2 C_\sigma^{2L} F^{2L-2}$ is the stability constant and  scalar $\alpha$ is either the maximal node degree or $2$ depending on the shift operator [cf. Lemma \ref{lemma:traceOperation}].
\end{theorem}
\begin{proof}
See \ref{proof:theorem2}.
\end{proof}

Theorem \ref{theorem:GNNstability} shows the GCNN is stable to stochastic graph perturbations, and the result holds uniformly for any graph of $n$ nodes. The expected deviation in the GCNN output over random graphs is upper bounded by a term that depends linearly on the link loss probability $1-p$. When the link sampling probability $p \to 1$, the stochastic perturbations are small and the bound reduces to zero, i.e, the GCNN maintains its performance. While the bound may have a similar form as that for the filter [Thm. \ref{theorem:filterStability}], the stability constant $C$ incorporates the effects of different architectural components. In particular, it is the product of three terms:

\smallskip
\begin{enumerate}[(1)]

\item The first term $n \alpha$ captures the impact of the graph. A larger graph leads to a looser bound, which implies the GCNN may be less stable in the large-scale setting. This is because a larger graph introduces more link losses with the same link sampling probability $p$ leading to an increased effect. Also a larger $\alpha$ leads to a worse bound, whose value is determined by the choice of the graph shift operator.

\item The term $C_L^2$ captures properties of the generalized graph filter frequency response $h(\bblambda)$ and the latter is determined by the filter coefficients. Graph filters with a larger generalized integral Lipschitz constant $C_L$ [cf. Def. \ref{def:GeneralizedIntegralLIpschitz}] allow more variability of $h(\bblambda)$ between nearby multivariate frequencies and, thus, are less stable to frequency deviations induced by stochastic perturbations. Reducing $C_L$ yields graph filters whose generalized frequency response $h(\bblambda)$ changes more slowly, increasing the robustness while reducing the discriminatory power.

\item The last term $L^2 C_\sigma^{2L} F^{2L-2}$ represents the impact of the GCNN architecture. It captures the effects of stochastic perturbations passing through the multiple filter banks ($F$), nonlinearities ($C_\sigma$), and layers ($L$). In particular, $C_\sigma$ is typically one implying the non-expansivity of the nonlinearity. A wider architecture with more features per layer and a deeper architecture with more layers lead to a looser bound. Both can be explained by the fact that more filters with stochastic perturbations are involved in the architecture, ultimately, increasing the randomness of the GCNN output.

\end{enumerate}

Note that Theorems \ref{theorem:filterStability} and \ref{theorem:GNNstability} establish the stability results by analyzing the mean square error, which is the second moment information. The latter can also be used to identify the probabilistic contraction bound for filter/GCNN output deviations induced by stochastic graph perturbations, as shown by the following corollary.

\begin{corollary}\label{Corollary2}
Consider the same settings as Theorem \ref{theorem:GNNstability}. For any graph signal $\bbx$, the deviation of stochastically perturbed GNN output satisfies
 \begin{equation}\label{eq:varianceCharacterizeGCNN}
\text{Pr}\left[ \| \tilde{\bbPhi}(\bbx;\bbS,\ccalH) - \bbPhi(\bbx;\bbS,\ccalH)\|_2^2  \le \epsilon \right] \ge 1- \frac{C(1-p)\|\bbx\|_2^2 }{\epsilon} - \ccalO((1-p)^2)
\end{equation}
where $\text{Pr}[\cdot]$ is the probability and $C$ is the stability constant in Theorem \ref{theorem:GNNstability}.
\end{corollary}

\begin{proof}
See \ref{proof:corollary2}.
\end{proof}

\noindent Corollary \ref{Corollary2} quantifies how much the stochastically perturbed GCNN output deviates from the unperturbed GCNN output\footnote{The similar result can be obtained for the graph filter from Theorem \ref{theorem:filterStability}.}. Specifically, it establishes that the probability that the stochastically perturbed output deviates from the unperturbed output by more than $\epsilon$ is at most a fraction of $C(1-p)\|\bbx\|_2^2 /\epsilon$. This result provides an explicit guarantee for the deviation distribution of the GCNN output induced by stochastic perturbations.

In general, Theorems \ref{theorem:filterStability} and \ref{theorem:GNNstability} not only characterize the stability of graph filters and GCNNs respectively, but also indicate the impact of the graph stochasticity, the filter property, the nonlinearity, and the architecture width and depth. These theoretical results provide insights on which factors to consider when designing a GCNN architecture with improved robustness under stochastic perturbations. Specifically, the nonlinearity $\sigma(\cdot)$, the architecture width $F$, and the architecture depth $L$ are design choices. A wider GCNN with more features or a deeper GCNN with more layers results in a worse stability but improves the representational power. One could decide how wide or deep the GCNN should be to achieve a trade-off between these two factors. Moreover, we see that the number of layers $L$ is the most critical factor that affects the bound. If a robust GCNN is needed, an architecture with less layers and more features may be preferred compared with an architecture with more layers and less features.

Lastly, we remark that the stability properties are built upon the generalized integral Lipschitz property [cf. \eqref{eq:GeneralizedIntegralLipschitzFilter}]. The latter is important because it restricts the change of the generalized frequency response $h(\bblambda)$ induced by the change of multivariate frequency $\bblambda$, especially for large-value frequencies. With the following example we aim to illustrate the latter in more detail.

\textbf{Example 1.} Consider a sequence of shift operators $\bbS_1,\ldots,\bbS_K$ characterized by the RES($\ccalG, p$) model with expected value
\begin{equation}\label{eq:IntegralDiscuss1}
\mathbb{E}[\bbS_k] = p \bbS~\forall~k=1,\ldots,K.
\end{equation}
The expected spectrum of $\mathbb{E}[\bbS_k]$ depends on not only the link sampling probability $p$ but also the underlying shift operator $\bbS$. I.e., the eigenvalues of $\mathbb{E}[\bbS_k]$ are of the form $p \lambda_i$ for a given eigenvalue $\lambda_i$ of $\bbS$. The generalized frequency response instantiated on the sequence of expected graphs is $h(p \bblambda)$ instead of $h(\bblambda)$. The key point to establish Theorems \ref{theorem:filterStability}-\ref{theorem:GNNstability} is to characterize the difference $|h(\bblambda)-h(p\bblambda)|$. If $h(\bblambda)$ is only generalized Lipschitz --i.e., $\| \nabla_L h(\bblambda_1,\bblambda_2) \|_2 \le C_L$, which is easier to satisfy-- but not generalized integral Lipschitz [Def. \ref{def:GeneralizedIntegralLIpschitz}], this difference is upper bounded by
\begin{align}\label{eq:IntegralDiscuss2}
|h(\bblambda) - h(p\bblambda)| &\leq \|\nabla_L h(\bblambda, p \bblambda)\|_2 \|(1-p)\bblambda\|_2\\
&\leq (1-p) C_L \|\bblambda\|_2\nonumber
\end{align}
which is small when the link sampling probability is high (i.e., $p \approx 1$) and the multivariate frequency is small (i.e., $\| \bblambda \|_2$ is small). However, for large-value multivariate frequency $\bblambda$, the filter becomes instable even if there are a few stochastic perturbations. The generalized integral Lipschitz property solves this issue by tapering off as $\bblambda$ grows (i.e., as $\|\bblambda\|_2$ increases). More specifically, we have 
\begin{align}\label{eq:IntegralDiscuss3}
&|h(\bblambda) - h(p\bblambda)| = |(1-p)\bblambda^\top \nabla_L h(\bblambda, p\bblambda)| \\
& \le \| (1-p)\bblambda \odot \nabla_L h(\bblambda, p\bblambda)\|_2 \le (1-p)C_L \nonumber.
\end{align}
This recovers stability for large-value multivariate frequencies.

Thus, the generalized integral Lipschitz property plays an important role in the stability analysis to stochastic perturbations as its counterpart for the deterministic perturbation \cite{gama2020stability}.

\section{Numerical Simulations}\label{sec:numericalExperiments}

We corroborate the theoretical results on distributed source localization (Section \ref{exp:source}) and robot swarm control (Section \ref{exp:flocking}). In both cases, we evaluate the stability of the GCNN to stochastic graph perturbations characterized by the RES($\ccalG, p$) model.

\subsection{Source Localization} \label{exp:source}

The goal of this experiment is to determine the source of diffused graph signals distributively. We consider the signal diffusion process over a stochastic block model (SBM) graph, which contains $n=100$ nodes uniformly divided into $c=5$ communities. The intra-community link probability is $0.8$ and the inter-community link probability is $0.2$. The source signal is a Kronecker delta $\bbdelta_s = [\delta_1, \ldots, \delta_n]^\top \in \{ 0,1 \}^n$ where $\delta_s \neq 0$ at the source node $s \in \{ s_1, \ldots, s_5 \}$. The diffused signal at time $t$ is $\bbx_{st} = \bbS^t \bbdelta_s + \bbn$ with $\bbS = \bbA / \lambda_{max}(\bbA)$ the normalized adjacency matrix and $\bbn$ the normal noise. The dataset consists of $15,000$ samples generated randomly with a source node and a diffused time $t \in [0, 50]$, and is split into $10,000$ for training, $2,500$ for validation, and $2,500$ for testing. We considered two baseline architectures: a linear graph filter bank and a two-layer GCNN with ReLU nonlinearity. Both comprise $F=32$ parallel filters of order $K=5$. We used the ADAM optimizer for training with decaying factors $\beta_1 = 0.9, \beta_2 = 0.999$ and the learning rate $\mu = 10^{-3}$ \cite{Ba2010}. We measured the performance as the classification accuracy averaged over ten graph realizations and ten data splits.

\begin{figure}[t]
\centering
\includegraphics[width=0.6\linewidth , height=0.4\linewidth, trim=10 10 10 10]{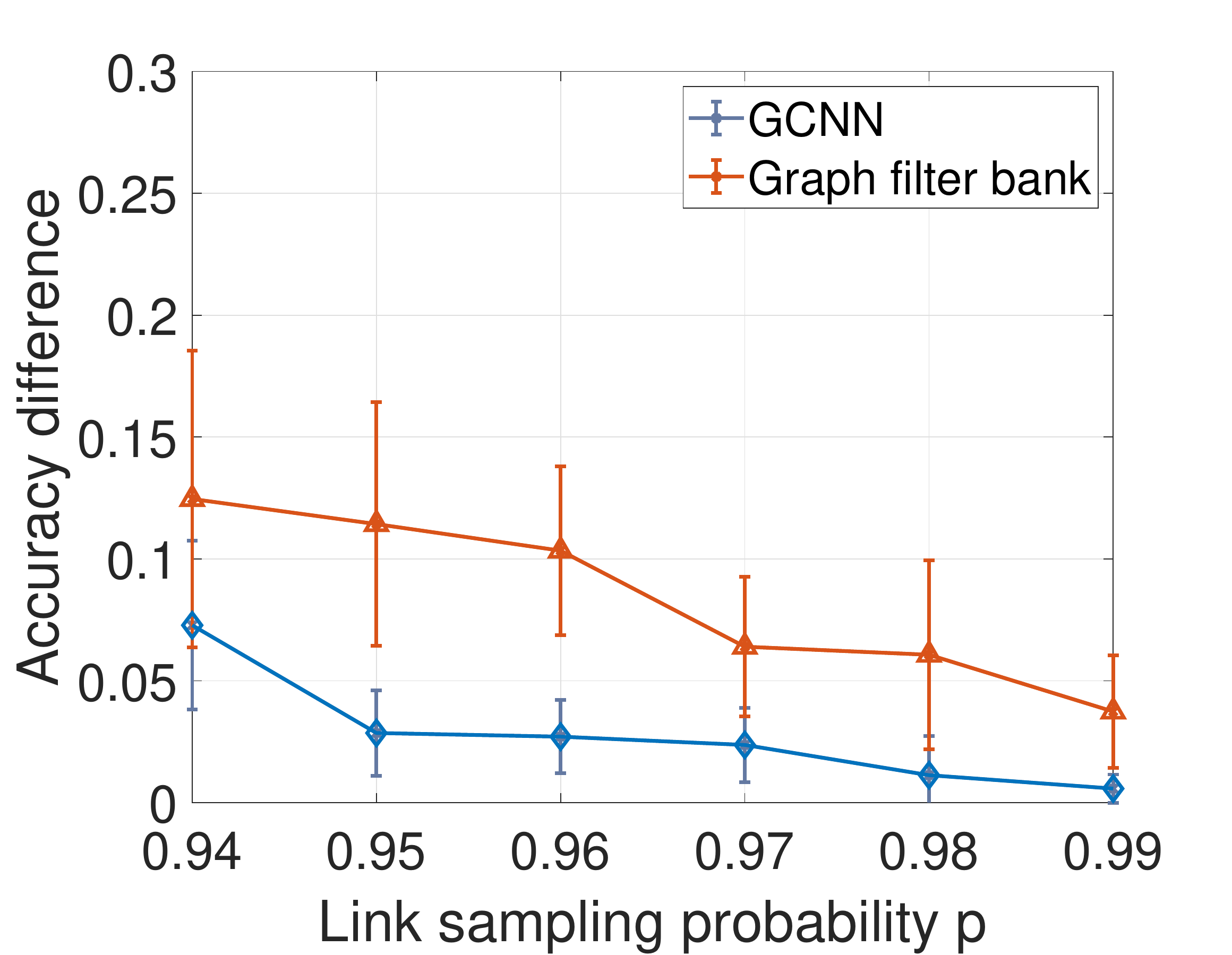}
\caption{Performance difference of the graph filter bank and the GCNN induced by stochastic perturbations with different link sampling probabilities $p$ in the source localization.}
\label{fig:source_different_p}
\end{figure}

\begin{figure*}%
\centering
\begin{subfigure}{0.48\columnwidth}
\includegraphics[width=1.1\linewidth,height = 0.85\linewidth]{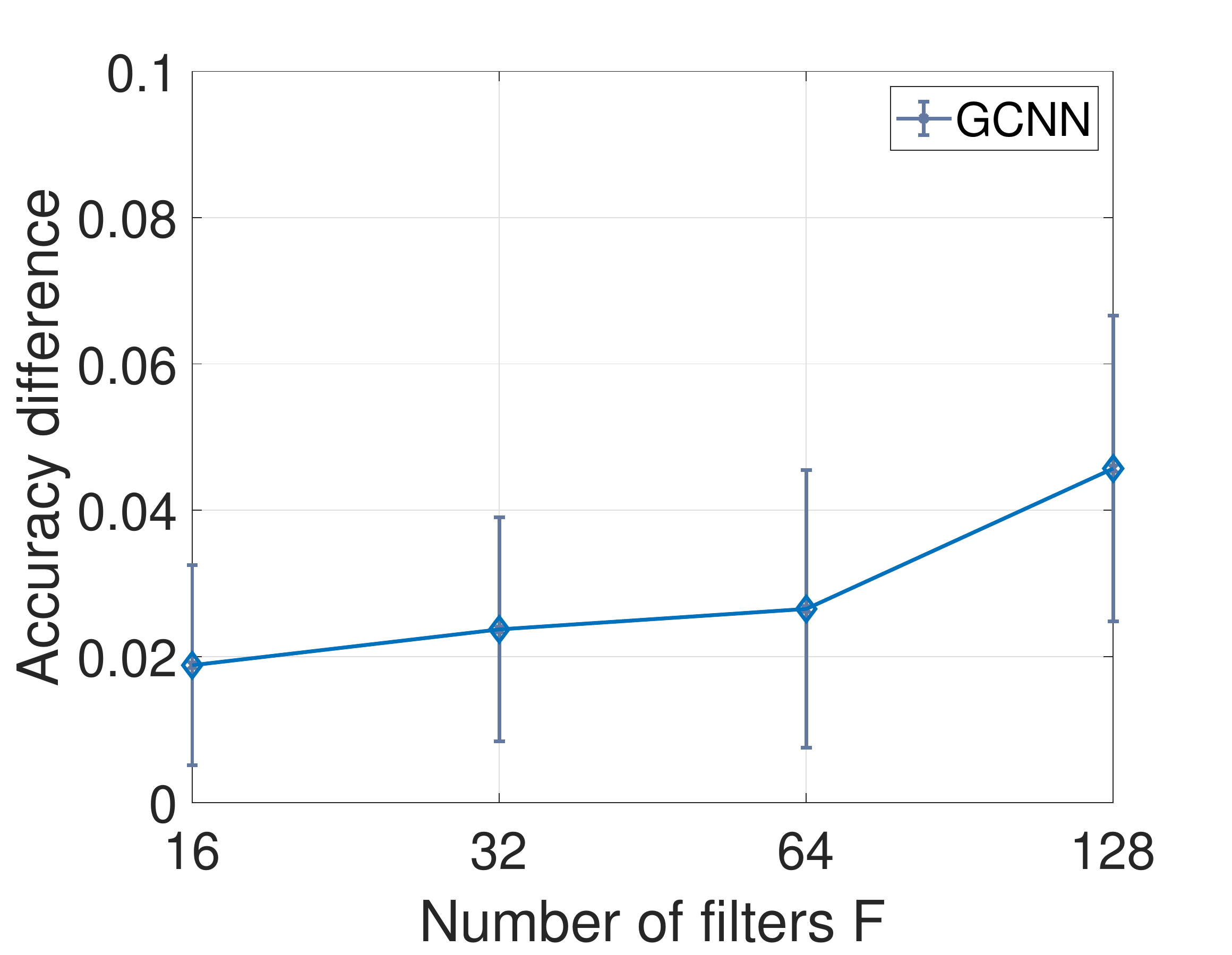}%
\caption{}%
\label{subfigb_vary_r}%
\end{subfigure}\hfill\hfill%
\begin{subfigure}{0.48\columnwidth}
\includegraphics[width=1.1\linewidth, height = 0.85\linewidth]{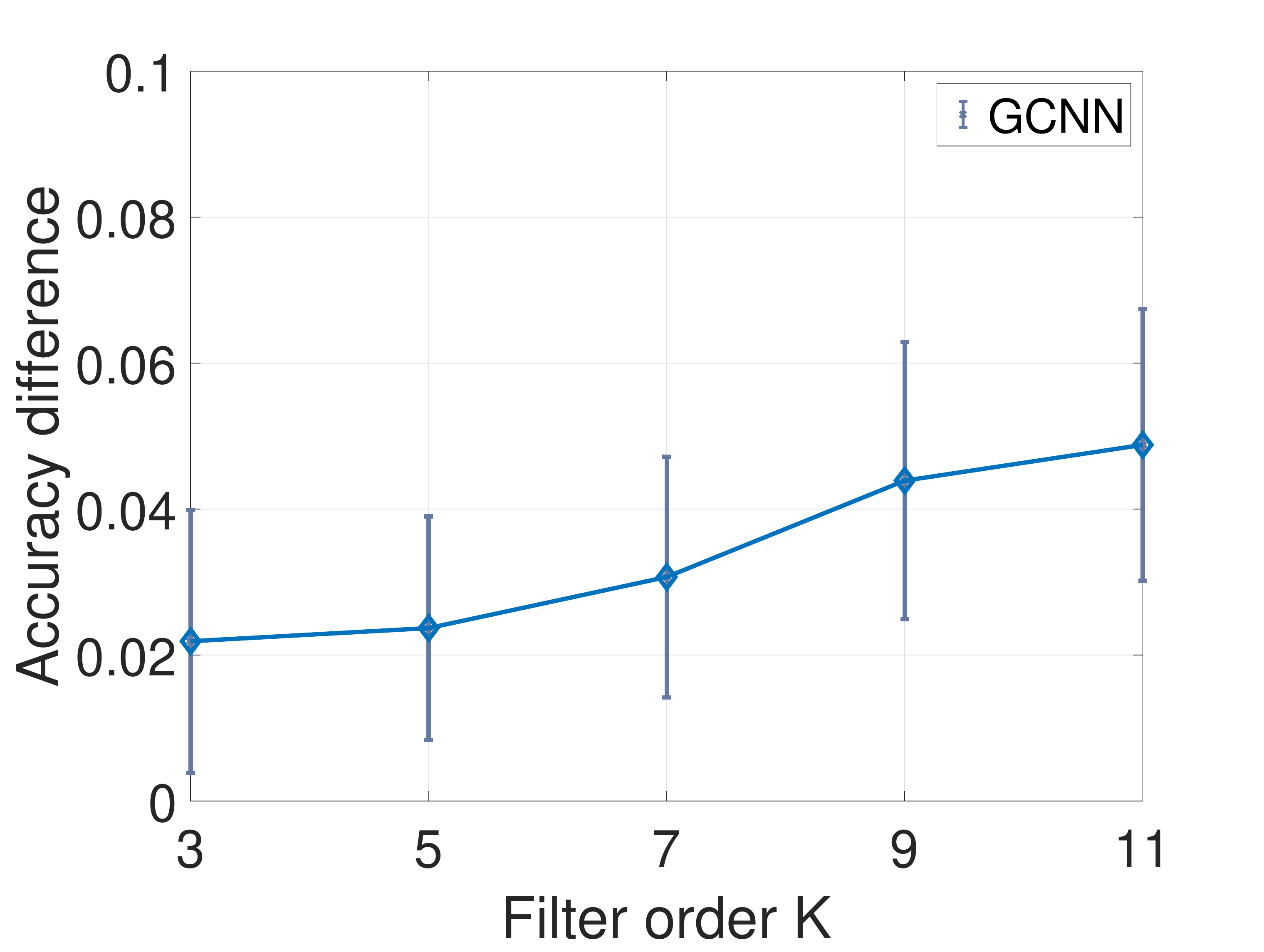}%
\caption{}%
\label{subfiga_vary_K}%
\end{subfigure}\hfill\hfill%
\begin{subfigure}{0.48\columnwidth}
\includegraphics[width=1.075\linewidth,height = 0.85\linewidth]{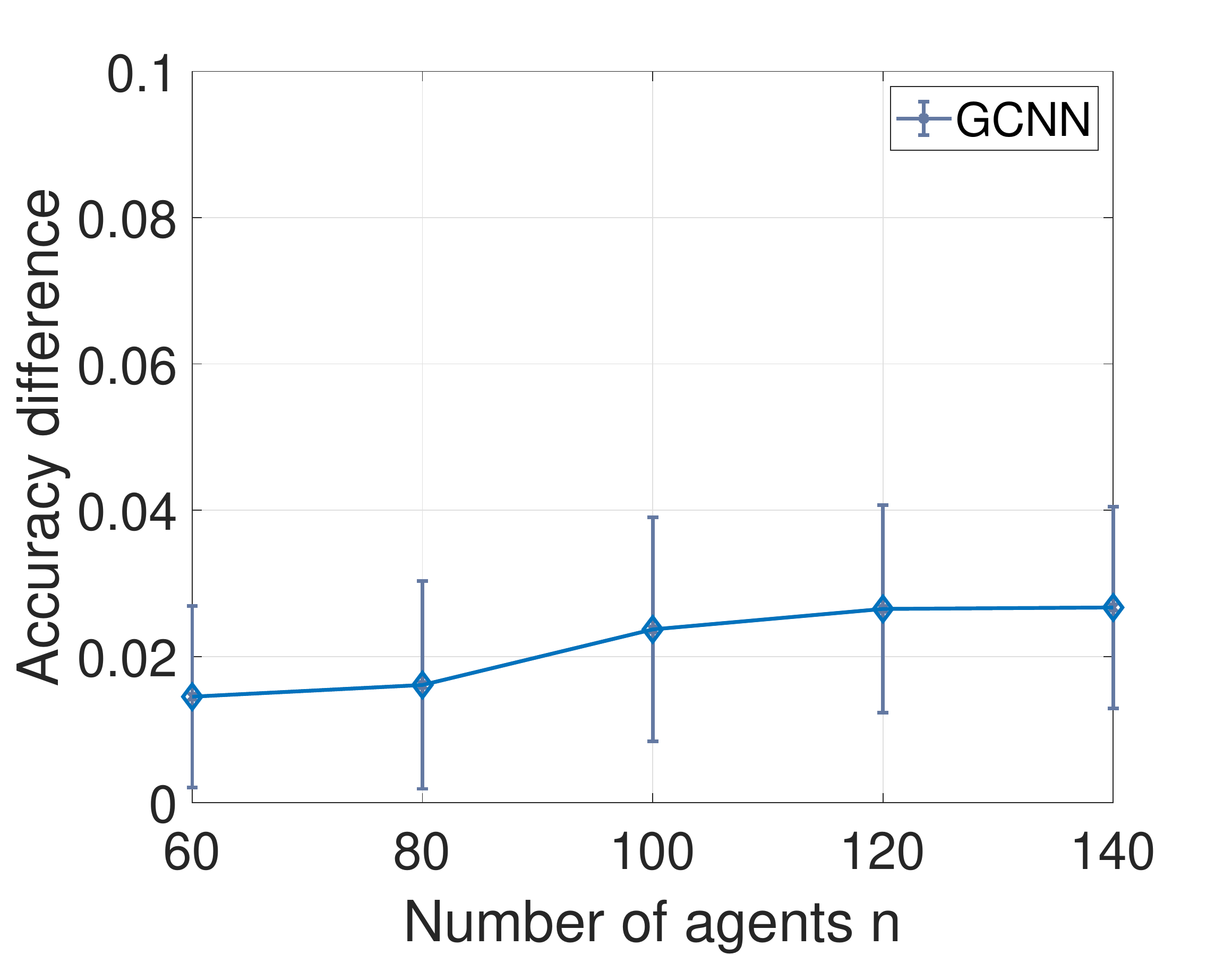}%
\caption{}%
\label{subfigc_vary_n}%
\end{subfigure}\hfill\hfill%
\begin{subfigure}{0.48\columnwidth}
\includegraphics[width=1.1\linewidth,height = 0.85\linewidth]{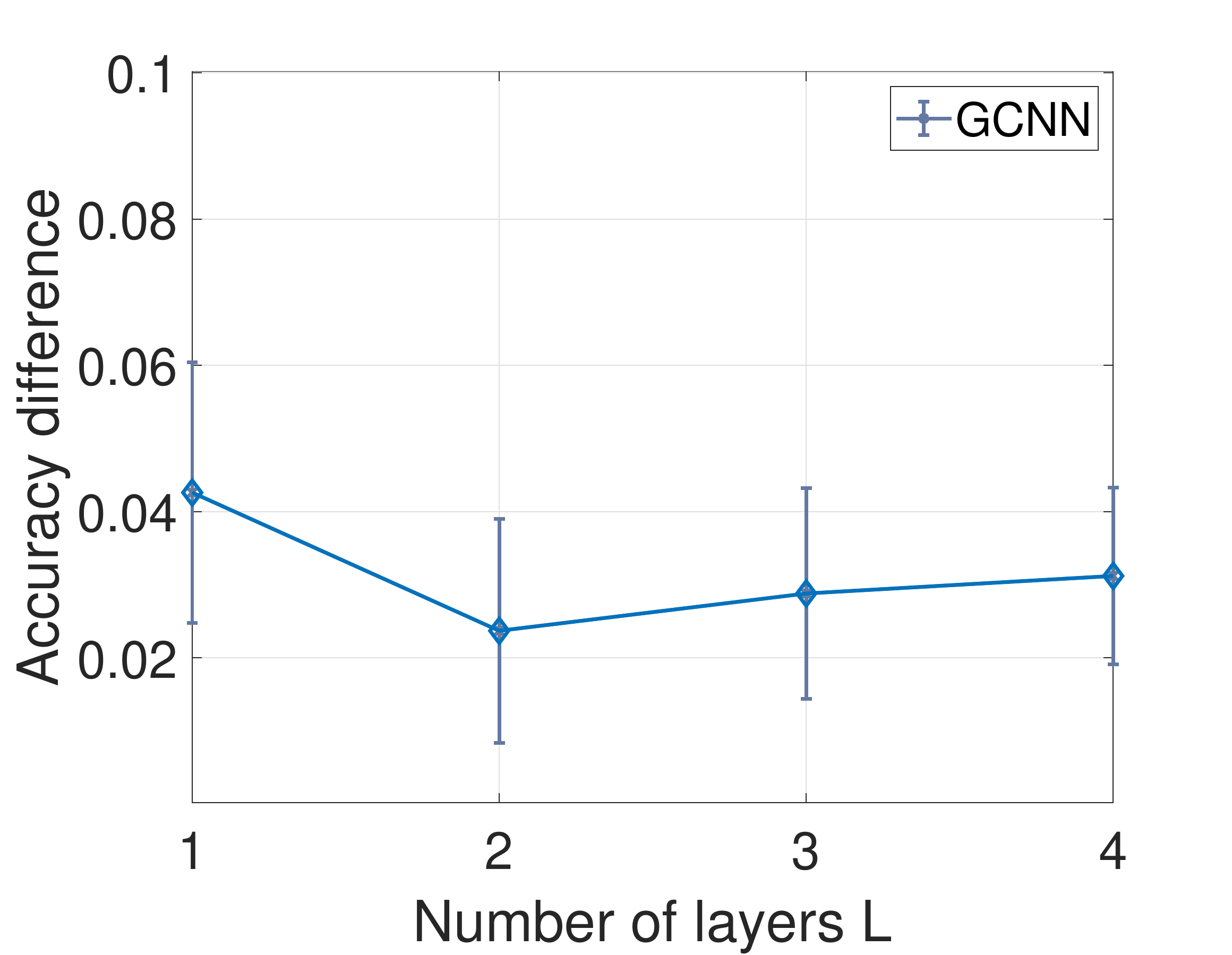}%
\caption{}%
\label{subfigb_vary_L}%
\end{subfigure}%
\caption{Performance difference induced by stochastic perturbations in the source localization. (a) Comparison under different number of filters $F$. (b) Different filter order $K$. (c) Different number of agents $n$. (d) Different number of layers $L$. For reference, the expected performance of the filter bank and the GCNN on the nominal graph are $0.86$ and $0.94$, respectively.}\label{fig:vary_n1}
\end{figure*}

First, we analyzed the effects of the link sampling probability $p \in [0.94, 0.99]$, which represents the severity of stochastic perturbations. Fig. \ref{fig:source_different_p} shows the expected performance difference of the filter bank and the GCNN. We see that when the link sampling probability $p$ approaches one, there is little performance degradation. Instead when $p$ becomes smaller and more links get lost, the performance of both architectures degrades. Furthermore, the GCNN exhibits a stronger stability with a lower expected difference and standard deviation compared with the filter bank. We attribute this behavior to two aspects: (i) the Lipschitz nonlinearity is nonexpansive when propagating the stability result [cf. \eqref{eq:LipschitzNonlinear}]; (ii) the filter bank performs worse with lower classification accuracy such that its performance may be more susceptible to the architecture output difference induced by stochastic perturbations.

Secondly, we analyzed the stability of the GCNN for different problem settings\footnote{The graph filter bank exhibits similar behaviors as the GCNN in these settings and we do not show them for repetition avoidance.}, i.e., different number of filters $F$ in Fig. \ref{subfigb_vary_r}, different filter order $K$ in Fig. \ref{subfiga_vary_K}, and different number of nodes $n$ in Fig. \ref{subfigc_vary_n}. The link sampling probability is set to $p = 0.97$. Fig. \ref{subfigb_vary_r} and Fig. \ref{subfiga_vary_K} show the expected performance difference increases with the number of filters $F$ and the filter order $K$. This corroborates the hypothesis that more filters with higher filter orders contain more graph shifts, which amplifies the impact of stochastic perturbations on the performance of the GCNN. In Fig. \ref{subfigc_vary_n}, we see the expected performance difference increases with the number of nodes $n$ and the increasing rate decreases for larger $n$. This is because stochastic perturbations tie to the underlying graph [cf. \eqref{eq:IntegralDiscuss1}] such that a larger graph leads to more link losses and larger performance difference. When the network size increases further, the graph becomes more connected and more robust to information exchanges between nodes. The GCNN learns stronger representations whose performance is less susceptible to the output difference and the latter slows down the increasing rate.

Lastly, we compared the stability of GCNNs with different numbers of layers $L$ in Fig. \ref{subfigb_vary_L}. The number of filters per layer is $F=32$ and the link sampling probability is $p=0.97$. We see the expected performance difference first decreases for $L$ from $1$ to $2$ and then increases for $L\ge 2$. The former is because the one-layer GCNN has a worse performance that may be more susceptible to the output difference caused by stochastic perturbations. The results for $L \ge 2$ correspond to the findings in Thm. \ref{theorem:GNNstability} that a deeper GCNN may lead to a less stable architecture.

\subsection{Robot Swarm Control}\label{exp:flocking}

\begin{figure*}%
\centering
\begin{subfigure}{0.48\columnwidth}
\includegraphics[width=1.1\linewidth, height = 0.85\linewidth]{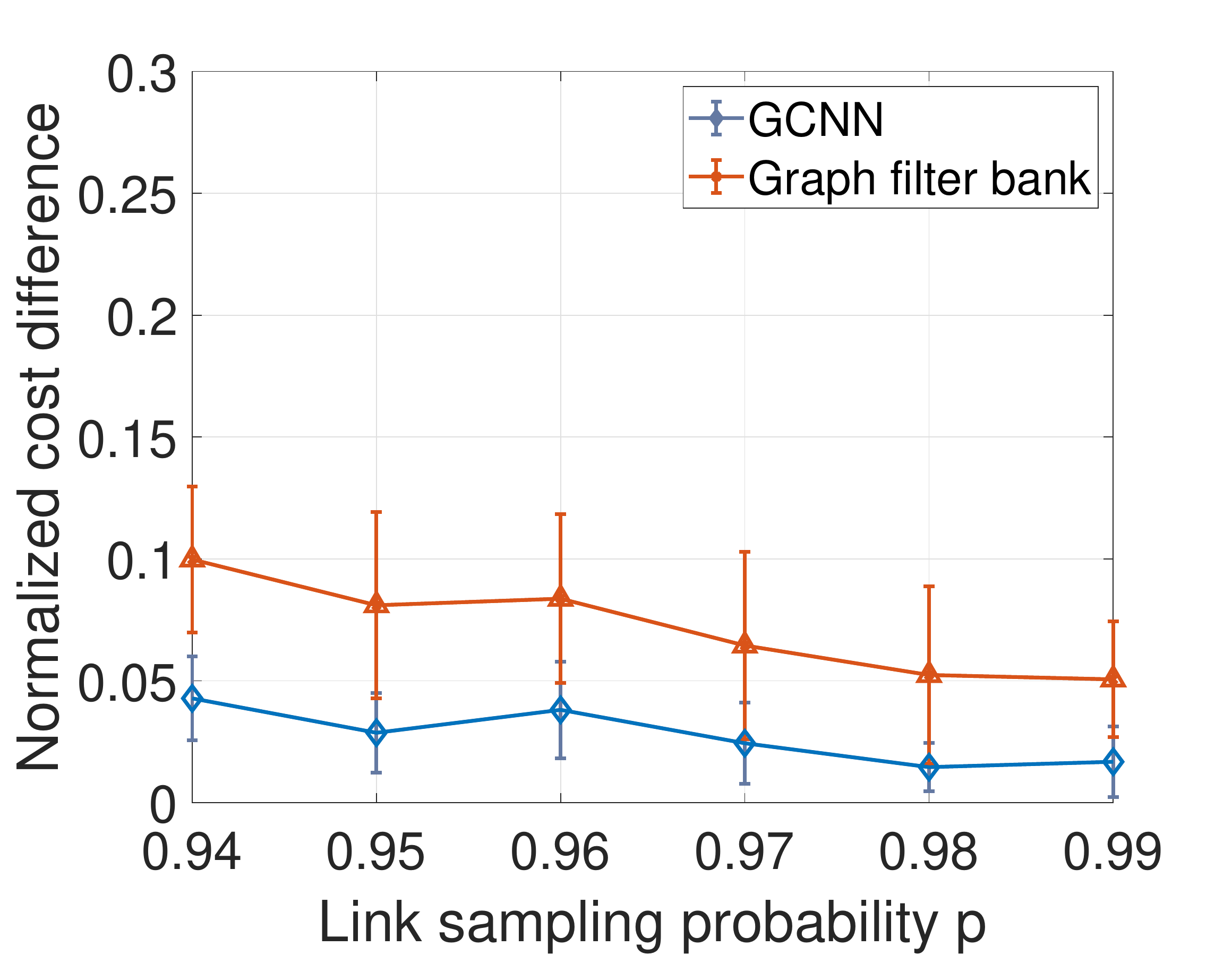}%
\caption{}%
\label{subfig6_vary_p}%
\end{subfigure}\hfill\hfill%
\begin{subfigure}{0.48\columnwidth}
\includegraphics[width=1.1\linewidth,height = 0.85\linewidth]{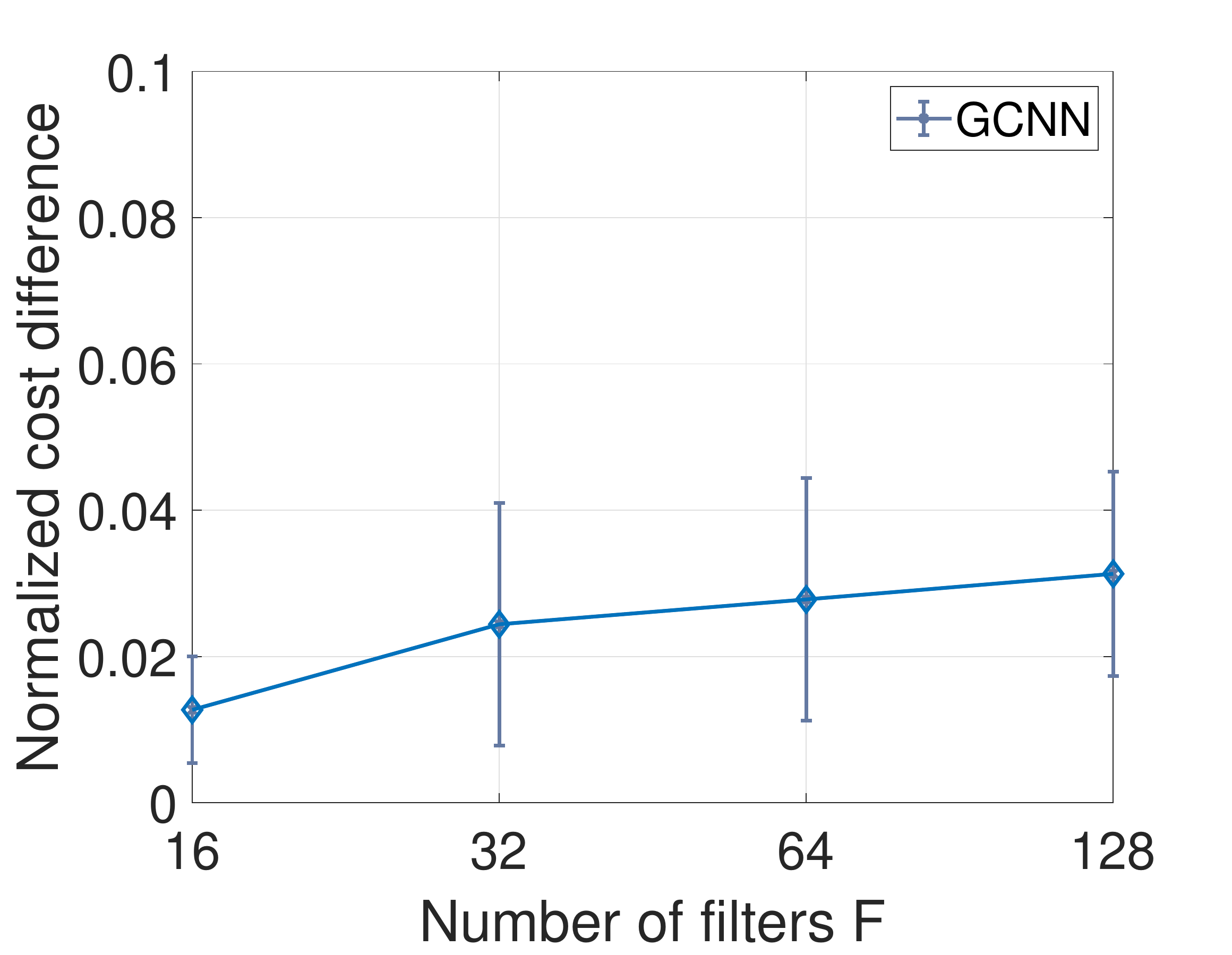}%
\caption{}%
\label{subfig6_vary_r}%
\end{subfigure}\hfill\hfill%
\begin{subfigure}{0.48\columnwidth}
\includegraphics[width=1.1\linewidth,height = 0.85\linewidth]{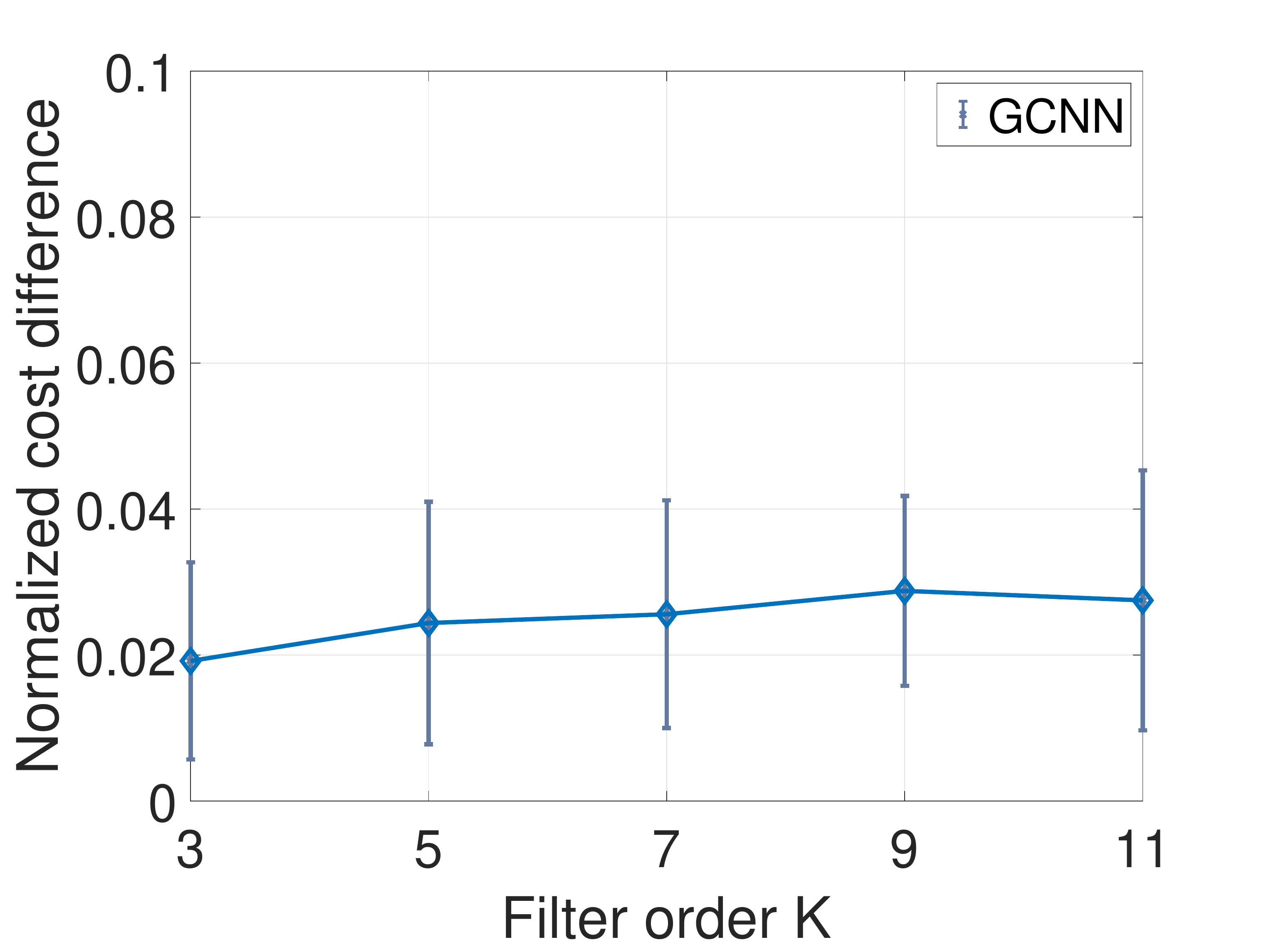}%
\caption{}%
\label{subfig6_vary_K}%
\end{subfigure}\hfill\hfill%
\begin{subfigure}{0.48\columnwidth}
\includegraphics[width=1.075\linewidth,height = 0.85\linewidth]{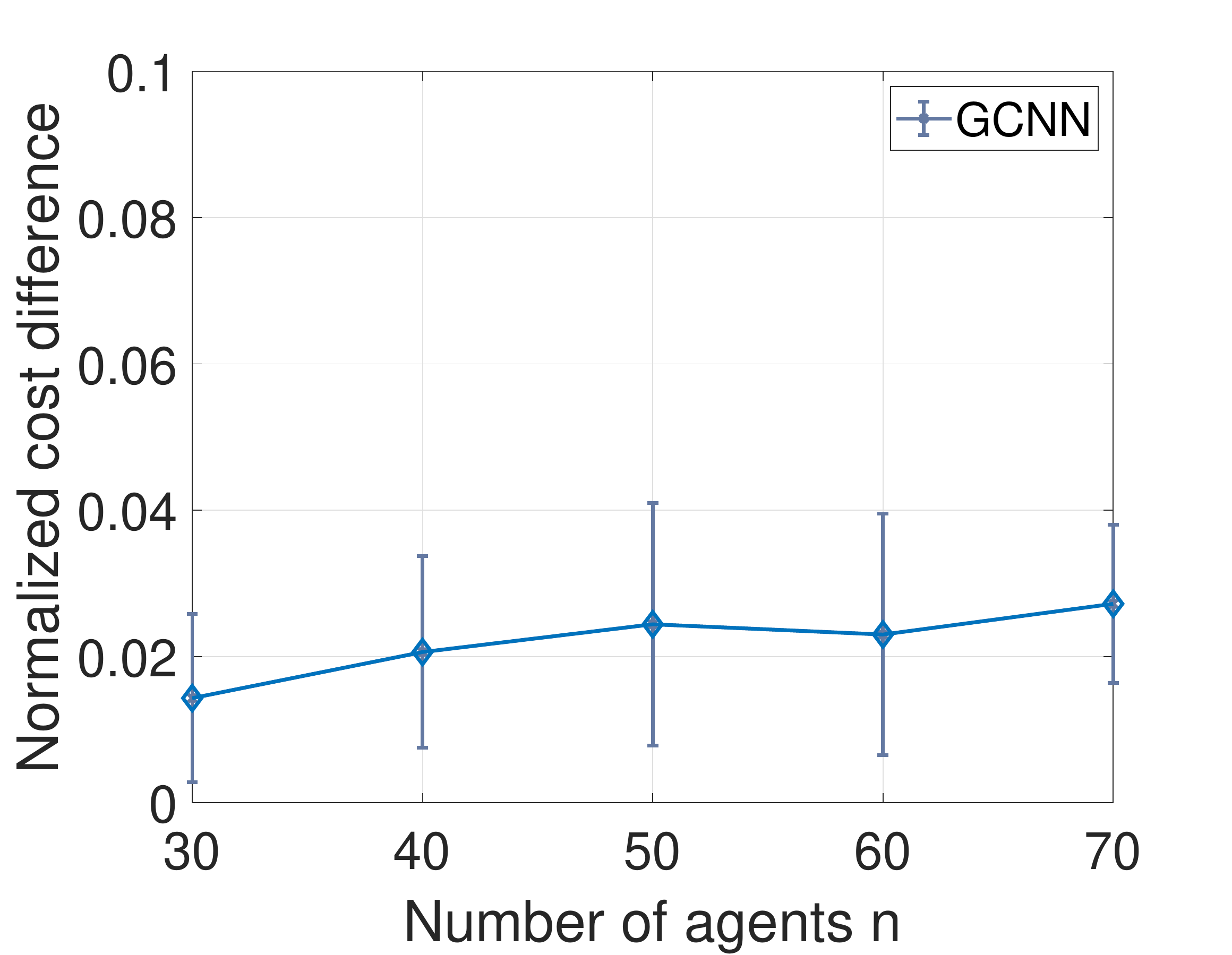}%
\caption{}%
\label{subfig6_vary_n}%
\end{subfigure}%
\caption{Normalized cost difference induced by stochastic perturbations in the robot swarm control. (a) Comparison between the GCNN and the graph filter bank under different link sampling probabilities $p \in [0.94,0.99]$. (b) Different number of filters $F$. (c) Different filter order $K$. (d) Different number of agents $n$. For reference, the expected cost of the filter bank and the GCNN without stochastic perturbations are $397.87$ and $88.74$, respectively.}\label{fig:vary_n1}
\end{figure*}

In this section, we considered the robot swarm control from \cite{gama2020graphs}. The goal is to learn a distributed controller that coordinates agents to move at the same velocity while avoiding collision. The network consists of $n=50$ agents with initial velocities sampled randomly in $[-3m/s, 3m/s]$. This task has a closed-form optimal solution that is centralized requiring the information of all agents. We aim to learn a distributed controller with the GCNN, which coordinates agents with only local neighborhood information. 

We consider agent $i$ can communicate with agent $j$ if they are within the communication radius of $r=2$m. The communication graph is $\ccalG(t) = \{ \ccalV, \ccalE(t) \}$ with agents $\ccalV = \{ 1, \ldots, n \}$ and communication links $\ccalE(t)$. The graph signal $\bbX(t)$ is the relevant feature of agent positions and velocities at time $t$ \cite{tolstaya2020learning}. The dataset contains $450$ trajectories of $100$ time steps, each of which is generated by initially positioning agents randomly in a circle with a minimum separation of $0.1$m, and is split into $400$, $25$ and $25$ samples for training, validation and test. We considered the linear graph filter bank and the one-layer GCNN with hyperbolic tangent nonlinearity\footnote{We only consider the one-layer GCNN in this experiment because multiple-layer GCNNs need to repeat agent communications in each layer [cf. \eqref{eq:GNNArchi}], which is not practical due to agent movements and graph changes \cite{tolstaya2020learning}.}. In both cases, we consider $F=32$ parallel filters of order $K=5$. We used the imitation learning to train the architectures by mimicing the centralized controller. The ADAM optimizer is used with decaying factors $\beta_1 = 0.9, \beta_2 = 0.999$ and the learning rate $\mu = 5 \cdot 10^{-4}$. We measure the controller performance as the velocity variance of agents throughout the trajectory \cite{Xiao2007}, and define the normalized performance difference $\gamma$ as the ratio of the cost difference $\Delta \ccalC$ induced by stochastic perturbations and the unperturbed cost $\ccalC$
\begin{equation}\label{eq:optimalController}
\begin{split}
\gamma = \frac{\Delta \ccalC}{\ccalC}.
\end{split}
\end{equation}
Our results are averaged over ten simulations.

In Fig. \ref{subfig6_vary_p}, we depict the expected normalized performance difference and the standard deviation for different link sampling probabilities $p \in [0.94, 0.99]$. 
We observe that when $p$ gets close to one, both the GCNN and the filter bank maintain the performance with little degradation. The degradation becomes more visible for lower link sampling probabilities $p$, while the GCNN achieves a stronger stability with a lower expected difference and standard deviation than the filter bank. We again attribute this behavior to the nonexpansivity of the nonlinearity and the improved performance of the GCNN.

We further compare the stability of the GCNN in different scenarios with a link sampling probability $p=0.97$. Fig. \ref{subfig6_vary_r}, Fig. \ref{subfig6_vary_K}, and Fig. \ref{subfig6_vary_n} analyze the impact of the number of filters $F$, the filter order $K$, and the number of agents $n$. When the number of filters $F$ and the filter order $K$ increase, the effects of link losses are more enhanced since more filters with higher filter orders add up the stochastic perturbations. When the number of agents $n$ increases, the performance difference increases with a decaying rate. This behavior is similar as in the previous experiment, i.e., a larger graph implies more link losses leading to a worse degradation, while it facilitates the information exchange between agents; hence, the learned architecture achieves better performance that may be less sensitive to the output difference. 

Overall, we remark the performance difference keeps relatively small values even for large $F$ and $n$, which indicates the GCNN maintains a strong stability.

\section{Conclusions}\label{sec:conclusions}

We studied the impact of stochastic perturbations on graph convolutional neural networks. We discussed first the stability of the graph filter by generalizing the graph spectral analysis over a succession of random graphs, and proved the expected output differs from the deterministic output by a factor that is upper bounded linearly by the link loss probability. We then showed the GCNN is stable to stochastic perturbations. The result indicates the explicit role played by the graph stochasticity, filter property, nonlinearity, and architecture width and depth on the stability of the GCNN. We found out higher Lipschitz constants of the filter and nonlinearity result in a worse stability. The same holds also for GCNNs containing more filters and layers. These theoretical insights help to establish the limitations of current solutions and identifying handle to improve the robustness. Numerical simulations corroborated our theoretical findings and showed that the GCNN maintains its performance under mild stochastic perturbations. This paper focused on the RES($\ccalG, p$) model motivated by practical distributed applications over physical networks. Interesting topics for future works include stability results for models involving edge additions or edge correlations.

\appendix 





\section{Proof of Lemma 1} \label{proof:lemmaLischitz}

\begin{proof} 
The generalized frequency response $h(\bblambda)$ is defined as
\begin{align}\label{proof:lemmaLischitz1}
h(\bblambda) = \sum_{k=0}^K h_k \prod_{\kappa=0}^k \lambda_k
\end{align}
with $\lambda_0 = 1$ by default [Def. \ref{def:IntegralLIpschitz}]. The partial derivatives $\big\{\frac{\partial h(\bblambda^{(k)})}{\partial \lambda_k}\big\}_{k=1}^K$ that define the Lipschitz gradient $\nabla_L h(\bblambda_1, \bblambda_2)$ are
\begin{align}
\frac{\partial h(\bblambda^{(1)})}{\partial \lambda_1} \!=\!\! \sum_{k=1}^K h_k\!\! \prod_{\kappa=2}^k\!\! \lambda_{2\kappa},~ \frac{\partial h(\bblambda^{(2)})}{\partial \lambda_2} \!=\! \sum_{k=2}^K \!h_k \lambda_{11}\!\! \prod_{\kappa=3}^k\!\! \lambda_{2\kappa}, \cdots\!, \frac{\partial h(\bblambda^{(K)})}{\partial \lambda_K} \!=\! h_K\!\! \prod_{\kappa = 1}^{K-1}\!\! \lambda_{1\kappa}\nonumber
\end{align}
where $\prod_{\kappa=k_1}^{k_2} \lambda_{i \kappa} = 1$ by default if $k_1 > k_2$ for $i=1, 2$. We expand the inner product $\nabla_L^\top h(\bblambda_1, \bblambda_2) (\bblambda_1 - \bblambda_2)$ as
\begin{align}\label{proof:lemmaLischitz3}
&\nabla_L^\top h(\bblambda_1,\! \bblambda_2) (\bblambda_1 \!-\! \bblambda_2) \!=\! \frac{\partial h(\bblambda^{(1)})}{\partial \lambda_1} (\lambda_{11}\!-\!\!\lambda_{21}) \!+\! \cdots \!+\! \frac{\partial h(\bblambda^{(K)})}{\partial \lambda_K}(\lambda_{1K}\!-\!\!\lambda_{2K}).
\end{align}
To simplify \eqref{proof:lemmaLischitz3}, we first consider terms that involve the filter coefficient $h_1$ since there is no term involving the filter coefficient $h_0$. There is only one term in $\big(\partial h(\bblambda^{(1)})/\partial \lambda_1\big)(\lambda_{11}-\lambda_{21})$, and we have
\begin{align}\label{proof:lemmaLischitz4}
h_1 (\lambda_{11}-\lambda_{21}) = h_1 \lambda_{11} - h_1 \lambda_{21}.
\end{align}
We then consider terms that involve the filter coefficient $h_2$. There are two terms in $\big(\partial h(\bblambda^{(1)})/\partial \lambda_1\big)(\lambda_{11}-\lambda_{21})$ and $\big(\partial h(\bblambda^{(2)})/\partial \lambda_2\big)(\lambda_{12}-\lambda_{22}) $ and we have
\begin{align}\label{proof:lemmaLischitz5}
h_2 \lambda_{22}(\lambda_{11}-\lambda_{21}) + h_2 \lambda_{11}(\lambda_{12}-\lambda_{22}) = h_2 \lambda_{11} \lambda_{12} - h_2 \lambda_{21} \lambda_{22}.
\end{align}
Following the same procedure with respect to the filter coefficients $h_3, \ldots, h_K$, we obtain
\begin{align}\label{proof:lemmaLischitz6}
\nabla_L^\top h(\bblambda_1, \bblambda_2) (\bblambda_1 \!-\! \bblambda_2) \!=\! \sum_{k=1}^K h_k \prod_{\kappa=1}^k \lambda_{1\kappa} \!-\! \sum_{k=1}^K h_k \prod_{\kappa=1}^k \lambda_{2\kappa} = h(\bblambda_1) - h(\bblambda_2)
\end{align}
completing the proof.
\end{proof}


\section{Proof of Theorem 1} \label{proof:theorem1}

\begin{proof}

Let $\tilde{\bbu} = \tilde{\bbH}(\bbS)\bbx$ be the filter output over RES($\ccalG,p$) subgraph realizations $\{\bbS_k\}_{k=1}^K$ [cf. \eqref{eq:randomGraphFilter}] and $\bbu = \bbH(\bbS)\bbx$ be that over the underlying graph $\bbS$ [cf. \eqref{eq:graphFilter}]. The expected output difference between $\tilde{\bbu}$ and $\bbu$ is
\begin{align} \label{proof:thm11}
&\mathbb{E}\left[\| \tilde{\bbu} - \bbu \|^2_2\right]\!=\! \mathbb{E}\!\left[\tr \big(\tilde{\bbu}^\top\! \tilde{\bbu} \!+\! \bbu^\top\! \bbu \!-\! 2\tilde{\bbu}^\top \bbu \big)\right]
\end{align}
where $\tr(\cdot)$ is the trace operator. By adding and subtracting $\bbu^\top \bbu$ in the trace, we can rewrite \eqref{proof:thm11} as
\begin{align} \label{proof:thm12}
\mathbb{E}\!\left[\tr \big(\tilde{\bbu}^\top\! \tilde{\bbu} \!-\! \bbu^\top\! \bbu \big)\right]+ 2 \mathbb{E}\!\left[\tr\big( \bbu^\top\! \bbu \!-\! \tilde{\bbu}^\top \bbu \big)\right]
\end{align}
where the linearity of the expectation and the trace is used. We consider two terms in \eqref{proof:thm12} separately.

For the first term $\mathbb{E}\!\left[\tr \Big(\tilde{\bbu}^\top\! \tilde{\bbu} \!-\! \bbu^\top\! \bbu \Big)\!\right]$, we substitute the filter expression for $\tilde{\bbu}$ and $\bbu$ and use the symmetry of the shift operator $\bbS_k$ and $\bbS$ to write\footnote{Throughout this proof, we will use the shorthand notation $\sum_{a,b,c = \alpha, \beta, \gamma}^{A, B, C} (\cdot)$ to denote $\sum_{a = \alpha}^A\sum_{b = \beta}^B\sum_{c = \gamma}^C (\cdot)$ to avoid overcrowded expressions. When the extremes of the sum ($\alpha, \beta, \gamma$ or $A, B, C$) are the same, we will write directly the respective value.}
\begin{gather} \label{proof:thm13}
\begin{split}
&\mathbb{E}\big[\tr \big(\tilde{\bbu}^\top\! \tilde{\bbu} \!-\! \bbu^\top\! \bbu \big)\big] = \sum_{k, \ell=0}^K h_k h_\ell \big( \mathbb{E} \big[ \tr \big( \tilde{T}(k,\ell) \big) \big] -  \tr \big( T(k,\ell) \big) \big)
\end{split}
\end{gather}
with $\tilde{T}(k,\ell)= \bbS_{k:0} \bbx \bbx^{\top} \bbS_{0:\ell}$ and $T(k, \ell)= \bbS^k  \bbx \bbx^{\top} \bbS^\ell$, where $\bbS_{k:0} = \bbS_k \cdots \bbS_0$, $\bbS_{0:\ell} = \bbS_0\cdots\bbS_\ell$ are concise notations and $\bbS_0 = \bbI$. We denote by $\lceil k\ell \rceil = \max(k,\ell)$ and $\lfloor k\ell \rfloor = \min(k,\ell)$ to further simplify notation, and represent the random shift operator $\bbS_k$ as $\bbS_k = \bbS + \bbE_k$ with $\bbE_k$ the deviation of $\bbS_k$ from $\bbS$. Substituting these expressions into $\tilde{T}(k,\ell)$ and expanding the terms yields
\begin{align}
\label{proof:thm135} &\mathbb{E}\left[ \tilde{T}(k,\ell)\right] = \mathbb{E}\left[ (\bbS + \bbE_k)\cdots \bbx \bbx^\top \cdots (\bbS + \bbE_\ell)\right]\\
&\!=\! {\bbS}^{k}\bbx \bbx^\top\! {\bbS}^{\ell}  \!\!+\! \mathbb{E}\big[ (\bbS \!+\! \bbE_k)\!\cdots\! \bbx \bbx^\top \bbS^\ell \!\!-\! {\bbS}^{k}\bbx \bbx^\top\! {\bbS}^{\ell} \big]\!+\! \mathbb{E}\big[ \bbS^{k}\bbx \bbx^\top \cdots (\bbS \!+\! \bbE_\ell) \!-\! {\bbS}^{k}\bbx \bbx^\top\! {\bbS}^{\ell}\big] \nonumber\\
& + \mathbb{E}\Big[ \sum_{r=1}^{\lfloor k\ell \rfloor} \bbS^{k-r}\bbE_r \bbS^{r-1}\bbx \bbx^\top \bbS^{r-1}\bbE_r \bbS^{\ell-r}\Big] + \mathbb{E}\left[ \bbC_{k\ell}\right] \nonumber 
\end{align}
for $k,\ell \ge 1$. The first term in \eqref{proof:thm135} contains the maximal power of $\bbS$; the second term includes error matrices only expanded from left-side $\bbS_{k:0}=(\bbS \!+\! \bbE_k)\cdots (\bbS+\bbE_1)$; the third term includes error matrices only expanded from right-side $\bbS_{0:\ell}=(\bbS \!+\! \bbE_1)\cdots (\bbS+\bbE_\ell)$; the fourth term collects the cross-products that include two error matrices $\bbE_r$ with the same index, and the last term $\bbC_{k\ell}$ aggregates the sum of the remaining terms. We then rewrite \eqref{proof:thm135} as
\begin{align}
\label{proof:thm14} 
 \mathbb{E}\left[ \tilde{T}(k,\ell)\right]&= - {\bbS}^{k}\bbx \bbx^\top {\bbS}^{\ell}  \!+\! \mathbb{E}\!\left[ \bbS_{k:0}\bbx \bbx^\top \bbS^\ell \right]\!+\! \mathbb{E}\!\left[ \bbS^{k}\bbx \bbx^\top \bbS_{0:\ell}\right]\\
& + \mathbb{E}\Big[ \sum_{r=1}^{\lfloor k\ell \rfloor} \bbS^{k\!-\!r}\bbE_r \bbS^{r\!-\!1}\bbx \bbx^\top \bbS^{r\!-\!1}\bbE_r \bbS^{\ell\!-\!r}\Big] \!+\! \mathbb{E}\left[ \bbC_{k\ell}\right]. \nonumber 
\end{align}
By substituting $T(k,\ell) =  {\bbS}^{k}\bbx \bbx^\top {\bbS}^{\ell}$ and \eqref{proof:thm14} into \eqref{proof:thm13}, we have
\begin{align} \label{proof:thm15}
&\mathbb{E}\!\left[\tr \big(\tilde{\bbu}^\top\! \tilde{\bbu} \!-\! \bbu^\top\! \bbu \big)\right] = - 2 \mathbb{E}\!\left[\tr\big( \bbu^\top\! \bbu \!-\! \tilde{\bbu}^\top \bbu \big)\right] \\
& \!+\! \sum_{k=1}^K \sum_{\ell=1}^K h_k h_\ell \tr \Big( \mathbb{E}\Big[ \sum_{r=1}^{\lfloor k\ell \rfloor} \bbS^{k\!-\!r}\bbE_r \bbS^{r\!-\!1}\bbx \bbx^\top \bbS^{r\!-\!1}\bbE_r \bbS^{\ell\!-\!r}\Big] \Big) \!+\! \sum_{k=0}^K \sum_{\ell=0}^K h_k h_\ell \tr \left( \mathbb{E}\left[ \bbC_{k\ell}\right] \right) \nonumber 
\end{align}
where the second and third terms in \eqref{proof:thm14} are the same and become $\tilde{\bbu}$ when summed up over indexes $k$ and $\ell$. We then consider the three terms in \eqref{proof:thm15} separately. For this analysis, we will use the inequality
\begin{gather} \label{proof:thm16}
\begin{split}
{\rm tr} (\bbA \bbB) \le \frac{\| \bbA + \bbA^\top \|_2}{2}{\rm tr}(\bbB) \le \| \bbA \|_2 {\rm tr}(\bbB)
\end{split}
\end{gather}
that holds for any square matrix $\bbA$ and positive semi-definite matrix $\bbB$ \cite{Wang1986}.

$\textbf{First term.}$ The first term in \eqref{proof:thm15} is the opposite of the second term in \eqref{proof:thm12} such that it cancels out when substituted into \eqref{proof:thm12}.

$\textbf{Second term.}$ By bringing the trace inside the expectation and leveraging the trace cyclic property $\tr(\bbA\bbB\bbC) = \tr(\bbC\bbA\bbB) = \tr(\bbB\bbC\bbA)$, we can write the second term in \eqref{proof:thm15} as
\begin{align}
\label{proof:thm17}\mathbb{E}\Big[ \sum_{r=1}^K \tr \Big( \sum_{k, \ell=r}^K \! h_k h_\ell \bbE_r \bbS^{k+\ell-2r}\bbE_r \bbS^{r-1}\bbx \bbx^\top\! \bbS^{r-1} \Big)\Big]
\end{align}
which also rearranged the terms to change the sum limits. By using Lemma \ref{lemma:traceOperation} [cf. \eqref{lemma1:mainresults}], we can upper bound \eqref{proof:thm17} as
\begin{align}
\label{proof:thm11895}
&\mathbb{E}\Big[ \sum_{r=1}^K \tr \Big( \sum_{k, \ell=r}^K \! h_k h_\ell \bbE_r \bbS^{k+\ell-2r}\bbE_r \bbS^{r-1}\bbx \bbx^\top\! \bbS^{r-1} \Big)\Big] \\
& \le n \alpha C_L^2 \| \bbx \|_2^2 (1-p) +  n C_L^2 \| \bbx \|_2^2 (1-p)^2 \nonumber
\end{align}
with $\alpha = d$ if $\bbS$ is the adjacency matrix and $\alpha = 2$ if $\bbS$ is the Laplacian matrix.

$\textbf{Third term.}$ Matrix $\bbC_{k\ell}$ comprises the sum of the remaining expansion terms. Each of these terms is a quadratic form of error matrices $\bbE_{r_1}$, $\bbE_{r_2}$ with $r_1 \neq r_2$; i.e., it is of the form $f_1(\bbS)\bbE_{r_1} f_2(\bbS)\bbE_{r_1}f_3(\bbS)\bbE_{r_2} f_4(\bbS)\bbE_{r_2} f_5(\bbS)$ or the form $f_1(\bbS)\bbE_{r_1} f_2(\bbS)\bbE_{r_1}f_3(\bbS)\bbE_{r_2} f_4(\bbS)$ for some functions $f_1(\cdot), ..., f_5(\cdot)$ that depend on the shift operator $\bbS$ and error matrices $\bbE_1, \ldots, \bbE_K$. Each of these double-quadratic terms can be bounded by a factor containing at least two terms of $\{\mathbb{E}[\bbE_{r_1}^2], \mathbb{E}[\bbE_{r_2}^2], \mathbb{E}[\bbE_{r_1}], \mathbb{E}[\bbE_{r_2}]\}$ inside the trace (i.e., a similar expression as \eqref{proof:thm111} but with at least two error matrix terms). Since the generalized frequency response $h(\bblambda)$ is bounded, the coefficients $\{ h_k \}_{k=0}^K$ are also bounded. From the facts that $\| \bbS \|_2$ is bounded and $\mathbb{E}[\bbE_r] = (1-p)\bbS$ and the bound of the second term in \eqref{proof:thm15} [cf. \eqref{proof:thm11895}], we have
\begin{gather} \label{proof:thm124}
\begin{split}
 \mathbb{E} \Big[\sum_{k, \ell=0}^K h_k h_\ell \bbC_{k\ell} \Big] = \ccalO((1-p)^2)\|\bbx\|_2^2.
\end{split}
\end{gather}

Finally, substituting the bounds for the first, second and third terms into \eqref{proof:thm16} and altogether into \eqref{proof:thm12}, we obtain
\begin{align} \label{proof:thm125}
&\mathbb{E}\left[\| \tilde{\bbH}(\bbS)\bbx - \bbH(\bbS)\bbx \|^2\right] \!\le\! n \alpha C_L^2 (1-p) \| \bbx \|^2_2 + \ccalO((1-p)^2).
\end{align}
completing the proof.
\end{proof}


\begin{lemma}\label{lemma:traceOperation}
Consider the graph filter $\bbH(\bbS)$ [cf. \eqref{eq:graphFilter}] with underlying shift operator $\bbS$ and filter coefficients $\{ h_k \}_{k=0}^K$. Consider also the filter run over $RES(\ccalG,p)$ subgraph realizations $\bbS_k = \bbS + \bbE_k$ for $k=1,\ldots,K$ [cf. \eqref{eq:randomGraphFilter}]. Let the filter be generalized integral Lipschitz with constant $C_L$ [Def. \ref{def:GeneralizedIntegralLIpschitz}]. Then, for any graph signal $\bbx$, it holds that
\begin{align}
\label{lemma1:mainresults}
&\mathbb{E}\Big[ \sum_{r=1}^K \tr \Big( \sum_{k, \ell=r}^K \! h_k h_\ell \bbE_r \bbS^{k+\ell-2r}\bbE_r \bbS^{r-1}\bbx \bbx^\top\! \bbS^{r-1} \Big)\Big]\\
& \le n \alpha C_L^2 \| \bbx \|_2^2 (1-p) +  n C_L^2 \| \bbx \|_2^2 (1-p)^2 \nonumber
\end{align}
with $\alpha = d$ if $\bbS$ is the adjacency matrix and $\alpha = 2$ if $\bbS$ is the Laplacian matrix.
\end{lemma}
\begin{proof}
Since both matrices $\sum_{k, \ell=r}^K h_k h_\ell \bbE_r \bbS^{k+\ell-2r}\bbE_r = \big(\sum_{k=r}^K h_k \bbS^{k-r}\bbE_r\big)^\top \big(\sum_{k=r}^K h_k \bbS^{k-r}\bbE_r\big)$ and $\bbS^{r-1}\bbx \bbx^\top \bbS^{r-1}$ in the left hand-side of \eqref{lemma1:mainresults} are positive semi-definite, we use the Cauchy-Schwarz inequality $\tr(\bbA \bbB) \le \tr(\bbA)\tr(\bbB)$ \cite{Zhang1999} and write
\begin{align}
\label{proof:thm18} &\mathbb{E}\Big[ \sum_{r=1}^K \tr \Big( \sum_{k, \ell=r}^K \! h_k h_\ell \bbE_r \bbS^{k+\ell-2r}\bbE_r \bbS^{r-1}\bbx \bbx^\top\! \bbS^{r-1} \Big)\Big]\\
&\le\mathbb{E}\Big[ \sum_{r\!=\!1}^K\! \sum_{k, \ell=r}^K \!h_k h_\ell \tr\! \left( \bbE_r \bbS^{k+\ell-2r}\bbE_r\! \right)\! \tr\! \left( \bbS^{r\!-\!1}\bbx \bbx^\top \bbS^{r\!-\!1} \right)\!\Big].\nonumber
\end{align}
Given the eigendecomposition $\bbS = \bbV\bbLambda\bbV^\top$ with eigenvectors $\bbV = [\bbv_1,\ldots,\bbv_n]^\top$ and eigenvalues $\bbLambda = \text{diag}(\lambda_1, \ldots, \lambda_n)$, the GFT of the graph signal $\bbx$ is $\bbx = \sum_{i=1}^n \hat{x}_i \bbv_i$. By substituting the latter into $\tr\! \left( \bbS^{r-1}\bbx \bbx^\top \bbS^{r-1} \right)$, we get
 \begin{gather}\label{proof:thm19}
\tr\! \left( \bbS^{r-1}\bbx \bbx^\top \bbS^{r-1} \right) = \sum_{i=1}^n\! \hat{x}_i^2 \lambda_i^{2r-2} \tr\! \left( \bbv_i \bbv_i^\top\right) = \sum_{i=1}^n \hat{x}_i^2 \lambda_i^{2r-2}
\end{gather}
where $\tr(\bbv_i\bbv_i^\top) = 1$ for $i=1,\ldots,n$ is used due to the orthonormality of eigenvectors. By substituting \eqref{proof:thm19} into \eqref{proof:thm18} and using the trace cyclic property $\tr(\bbA\bbB\bbC) = \tr(\bbC\bbA\bbB) = \tr(\bbB\bbC\bbA)$, we can rewrite \eqref{proof:thm18} as
\begin{align} \label{proof:thm111}
& \sum_{i = 1}^N\hat{x}_i^2\! \sum_{r=1}^K\!\tr \Big( \sum_{k, \ell=r}^K h_k h_\ell \lambda_i^{2r-2} \bbS^{k+\ell-2r} \mathbb{E}\!\left[ \bbE_r^2\right] \Big).
\end{align}
Similarly as the result of Lemma 2 in \cite{Zhan2020}, we can prove that $\mathbb{E}\left[ \bbE_k^2 \right] = (1-p)^2\bbS^2+\beta p(1-p) \bbE$ with $\beta = 1$ and $\bbE = \bbD$ is the degree matrix if $\bbS$ is the adjacency matrix and $\beta = 2$ and $\bbE = \bbS$ if $\bbS$ is the Laplacian matrix. Substituting this result into \eqref{proof:thm111}, we can represent it as
\begin{align} \label{proof:thm1115}
& (1 - p)^2 \sum_{i = 1}^N\hat{x}_i^2\! \sum_{r\!=\!1}^K\!\tr \Big( \sum_{k, \ell=r}^K h_k h_\ell \lambda_i^{2r-2} \bbS^{k+\ell-2r+2} \bbI \Big)\\
&+ \beta p (1-p) \sum_{i = 1}^N\hat{x}_i^2\! \sum_{r\!=\!1}^K\!\tr \Big( \sum_{k, \ell=r}^K h_k h_\ell \lambda_i^{2r-2} \bbS^{k+\ell-2r} \bbE \Big) \nonumber.
\end{align}
We consider two terms in \eqref{proof:thm1115} separately.

\textbf{First term.} For the first term in \eqref{proof:thm1115}, we use the inequality \eqref{proof:thm16} (since $\bbI$ is positive semi-definite) to upper bound it by
\begin{align}
\label{proof:thm112}
& (1- p)^2\sum_{i = 1}^N\hat{x}_i^2 \big\| \sum_{r\!=\!1}^K\! \sum_{k, \ell=r}^K\! h_k h_\ell \lambda_i^{2r-2} \bbS^{k+\ell-2r+2} \big\|_2 \tr\! \left( \bbI \right).
\end{align}
We now consider the matrix norm in \eqref{proof:thm112}. For any matrix $\bbA$, a standard way to bound $\|\bbA\|_2$ is to obtain the inequality $\|\bbA\bba\|_2 \le A \|\bba\|_2$ that holds for any vector $\bba$ \cite{Meyer2000}. In this context, $A$ is the upper bound satisfying $\|\bbA\|_2 \le A$. Following this rationale, consider the GFT of $\bba$ over $\bbS$ as $\bba = \sum_{j=1}^n \hat{a}_j \bbv_j$ and we have
\begin{align}
\label{proof:thm113} & \!\big\| \!\sum_{r=1}^K\! \sum_{k, \ell=r}^K\!\! h_k h_\ell \lambda_i^{2r\!-\!2} \bbS^{k\!+\!\ell\!-\!2r\!+\!2} \bba \big\|_2^2\!=\! \sum_{j=1}^n\! \hat{a}_j^2 \big| \sum_{r=1}^K\! \sum_{k, \ell=r}^K\! \!h_k h_\ell \lambda_i^{2r-2}\! \lambda_j^{k+\ell-2r+2}\big|^2.
\end{align}
The expression inside the absolute value in \eqref{proof:thm113} can be linked to the Lipschitz gradient of the analytic generalized frequency response $h(\bblambda)$ in \eqref{eq:LipschitzGradient}. More specifically, let $\bblambda_j = [\lambda_j, \ldots,\lambda_j]^\top$ and $\bblambda_i = [\lambda_i,\ldots,\lambda_i]^\top$ be specific multivariate frequencies and $\bblambda^{(r)} = [\lambda_{i}, \ldots, \lambda_{i}, \lambda_{j}, \ldots, \lambda_{j}]^\top$ formed by concatenating the first $r$ entries of $\bblambda_j$ and the last $K-r$ entries of $\bblambda_i$. The partial derivative of $h(\bblambda)$ [cf. \eqref{eq:GeneralizedFrequencyResponse}] w.r.t. the $r$th entry $\lambda_r$ of $\bblambda^{(r)}$ is
\begin{align}
\label{proof:thm114}
\frac{\partial h(\bblambda^{(r)})}{\partial \lambda_r} \!=\! \sum_{k=r}^K h_k \lambda_{i}^r \lambda_j^{k-r},\! ~\forall~r\!=\!1,\!\ldots,\!K.
\end{align}
The Lipschitz gradient of $h(\bblambda)$ between $\bblambda_j$ and $\bblambda_i$ [Def. \ref{def:LipschitzGradient}] is
\begin{equation}\label{proof:thm115}
\nabla_L h(\bblambda_j, \bblambda_i) = \Big[\frac{\partial h(\bblambda^{(1)})}{\partial \lambda_1}, \ldots, \frac{\partial h(\bblambda^{(K)})}{\partial \lambda_K}\Big]^\top.
\end{equation}
We observe that the expression inside the absolute value in \eqref{proof:thm113} can be written in the compact form and upper bounded by
\begin{align}
\label{proof:thm116} &\big|\sum_{r\!=\!1}^K\! \sum_{k, \ell=r}^K\! h_k h_\ell \lambda_i^{2r-2} \lambda_j^{k+\ell-2r+2}\big| \!=\!\!  \sum_{r\!=\!1}^K \!\Big( \lambda_j \frac{\partial h(\bblambda^{(r)})}{\partial\lambda_{r}} \Big)^2 \!\!=\! \| \bblambda_j \!\odot\! \nabla_L h(\bblambda_j, \bblambda_i) \|^2_2 \le C_L^2
\end{align}
because of the generalized integral Lipschitz condition with constant $C_L$ [cf. \eqref{eq:GeneralizedIntegralLipschitzFilter}], where $\odot$ is the elementwise product. Using \eqref{proof:thm116} in \eqref{proof:thm113}, we can bound the matrix norm by $C_L^2$. Further substituting this result into \eqref{proof:thm113} and altogether into \eqref{proof:thm112}, we have
\begin{align}
\label{proof:thm118} &(1 - p)^2\sum_{i = 1}^N\hat{x}_i^2 \big\| \sum_{r\!=\!1}^K\! \sum_{k, \ell=r}^K\! h_k h_\ell \lambda_i^{2r-2} \bbS^{k+\ell-2r+2} \big\|_2 \tr\! \left( \bbI \right)\nonumber \\
& \qquad\le (1-p)^2n C_L^2 \sum_{i=1}^N \hat{x}_i^2 = n C_L^2 \| \bbx \|^2_2 (1-p)^2.
\end{align}

\textbf{Second term.} For the second term in \eqref{proof:thm1115}, if $\bbS$ is the adjacency matrix with $\beta = 1$ and $\bbE = \bbD$ the degree matrix, we use the inequality \eqref{proof:thm16} to bound it by
\begin{align}
\label{proof:thm1185} & p(1-p)\sum_{i = 1}^N\hat{x}_i^2 \big\|\! \sum_{r=1}^K\! \sum_{k, \ell=r}^K\! h_k h_\ell \lambda_i^{2r-2} \bbS^{k+\ell-2r} \big\|_2 \tr\! \left( \bbD \right).
\end{align}
We follow \eqref{proof:thm113}-\eqref{proof:thm116} to bound the matrix norm in \eqref{proof:thm1185} by using the Lipschitz gradient [cf. \eqref{eq:LipschitzGradient}] and the generalized integral Lipschitz condition [cf. \eqref{eq:GeneralizedIntegralLipschitzFilter}] as
\begin{align}
\label{proof:thm1186} &\big\| \sum_{r=1}^K\! \sum_{k, \ell=r}^K\! h_k h_\ell \lambda_i^{2r-2} \bbS^{k+\ell-2r} \big\|_2 \le \| \nabla_L h(\bblambda_j, \bblambda_i) \|^2_2 \le C_L^2.
\end{align}
Using \eqref{proof:thm1186} in \eqref{proof:thm1185} and the fact $\tr(\bbD) \le n d$ with $d$ the maximal degree of graph, we can bound \eqref{proof:thm1185} by $nd C_L^2 \| \bbx \|_2^2 (1-p)$. If $\bbS$ is the Laplacian matrix with $\beta = 2$ and $\bbE = \bbS$, we similarly upper bound the second term in \eqref{proof:thm1115} by
\begin{align}
\label{proof:thm1187} &\beta p(1\!-\!p)\!\sum_{i = 1}^N\hat{x}_i^2 \big\|\! \sum_{r\!=\!1}^K\! \sum_{k, \ell=r}^K\!\! h_k h_\ell \lambda_i^{2r-2} \bbS^{k\!+\!\ell\!-\!2r\!+\!1} \big\|_2 \tr\! \left( \bbI \right).
\end{align}
Also following \eqref{proof:thm113}-\eqref{proof:thm116} bounds the matrix norm as
\begin{align}
\label{proof:thm1188} &\big\| \sum_{r=1}^K \sum_{k, \ell=r}^K\! h_k h_\ell \lambda_i^{2r-2} \bbS^{k+\ell-2r+1} \big\|_2 \\
&\le\frac{1}{2}\Big(\| \nabla_L h(\bblambda_j, \bblambda_i) \|_2^2 + \| \bblambda_j \odot \nabla_L h(\bblambda_j, \bblambda_i) \|_2^2\Big) \le C_L^2 \nonumber
\end{align}
where the inequality of arithmetic and geometric means is used \cite{kazarinoff1961geometric}. We can then bound \eqref{proof:thm1187} by $2 n C_L^2 \| \bbx \|_2^2 (1-p)$. Together, we upper bound the second term in \eqref{proof:thm1115} as
\begin{align}
\label{proof:thm1189}
&\beta p(1-p)\sum_{i = 1}^N\hat{x}_i^2 \tr \Big(\sum_{r\!=\!1}^K\! \sum_{k, \ell=r}^K\! h_k h_\ell \lambda_i^{2r-2} \bbS^{k+\ell-2r} \bbE \Big) \le n \alpha C_L^2 \| \bbx \|_2^2 (1-p)
\end{align}
with $\alpha = d$ for the adjacency matrix and $\alpha = 2$ for the Laplacian matrix.

By substituting \eqref{proof:thm118} and \eqref{proof:thm1189} into \eqref{proof:thm1115}, we complete the proof as
\begin{align}
&\mathbb{E}\Big[ \sum_{r=1}^K \tr \Big( \sum_{k, \ell=r}^K \! h_k h_\ell \bbE_r \bbS^{k+\ell-2r}\bbE_r \bbS^{r-1}\bbx \bbx^\top\! \bbS^{r-1} \Big)\Big] \\
&\le n \alpha C_L^2 \| \bbx \|_2^2 (1-p) + n C_L^2 \| \bbx \|_2^2 (1-p)^2 \nonumber
\end{align}
with $\alpha = d$ if $\bbS$ is the adjacency matrix and $\alpha = 2$ if $\bbS$ is the Laplacian matrix.
\end{proof}


\section{Proof of Theorem 2} \label{proof:theorem2}

\begin{proof}
From the GCNN architecture in \eqref{eq:GNNArchi} and the Lipschitz condition of the nonlinearity [cf. \eqref{eq:LipschitzNonlinear}], the output difference can be upper bounded by
\begin{align} \label{eq:thm21}
\big\|\tilde{\bbPhi}(\bbx;\!\bbS,\!\ccalH) \!-\! \bbPhi(\bbx;\!\bbS,\!\ccalH)\big\|_2 \!&\!=\!\! \big\| \sigma\!\Big(\!\sum_{f\!=\!1}^{F}\!\! \tilde{\bbu}^f_{L\!-\!1}\!\!\Big)\!\! -\!\sigma\!\Big(\!\sum_{f\!=\!1}^{F}\!\! \bbu^f_{L\!-\!1}\!\!\Big)\!\big\|_2\!\! \le\! C_\sigma \big\|\! \sum_{f\!=\!1}^{F}\!\! \tilde{\bbu}^f_{L\!-\!1} \!-\!\! \sum_{f\!=\!1}^{F}\! \!\bbu^f_{L-1}\big\|_2.
\end{align}
Applying the triangular inequality, we have
\begin{equation} \label{eq:thm22}
\begin{split}
&\big\|\tilde{\bbPhi}(\bbx;\!\bbS,\!\ccalH) \!-\! \bbPhi(\bbx;\!\bbS,\!\ccalH)\big\|_2 \!\le\! C_\sigma\sum_{f=1}^{F} \| \tilde{\bbu}^f_{L-1}\! - \! \bbu^f_{L-1}\|_2 .
\end{split}
\end{equation}
We consider each term $\| \tilde{\bbu}^f_{L-1}\! - \! \bbu^f_{L-1}\|_2$ separately. Denote $\tilde{\bbu}_{L-1}^f = \tilde{\bbH}_{L}^{f}\tilde{\bbx}^f_{L-1}$ and $\bbu_{L-1}^f=\bbH_L^f\bbx_{L-1}^f$ as concise notations of $\tilde{\bbH}_{L}^{f}(\bbS)\tilde{\bbx}^f_{L-1}$ and $\bbH(\bbS)_{L}^{f}\bbx^f_{L-1}$. By adding and subtracting $\tilde{\bbH}_{L}^{f}\bbx^f_{L-1}$ inside the norm, we get
\begin{align} \label{eq:thm23}
\| \tilde{\bbu}^f_{L-1}\! - \! \bbu^f_{L-1}\|_2&=\| \tilde{\bbH}_{L}^{f}\tilde{\bbx}^f_{L-1}\! - \!\tilde{\bbH}_{L}^{f}\bbx^f_{L-1} \!+\! \tilde{\bbH}_{L}^{f}\bbx^f_{L-1}\!- \! \bbH_L^f\bbx_{L-1}^f\|_2 \\
& \le \| \tilde{\bbH}_{L}^{f}\big(\tilde{\bbx}^f_{L-1}\! - \!\bbx^f_{L-1}\big)\|_2 \!+\!  \| \tilde{\bbH}_{L}^{f}\bbx^f_{L-1}\!- \! \bbH_L^f\bbx_{L-1}^f\|_2 \nonumber
\end{align}
Since $|h(\bblambda)|\le 1$ and similar to result as Lemma 3 in \cite{Zhan2020}, we upper bound the first term in \eqref{eq:thm23} as
\begin{align} \label{eq:thm24}
\| \tilde{\bbH}_{L}^{f}\big(\tilde{\bbx}^f_{L-1}\! - \!\bbx^f_{L-1}\big)\|_2 &\le \| \tilde{\bbH}_{L}^{f} \|_2 \| \tilde{\bbx}^f_{L-1}\! - \!\bbx^f_{L-1} \|_2 \le \| \tilde{\bbx}^f_{L-1}\! - \!\bbx^f_{L-1} \|_2.
\end{align}
By substituting \eqref{eq:thm24} into \eqref{eq:thm23} and altogether into \eqref{eq:thm22}, we have
\begin{equation} \label{eq:thm25}
\begin{split}
\big\|\tilde{\bbPhi}(\bbx;\!\bbS,\!\ccalH) \!\!-\! \bbPhi(\bbx;\!\bbS,\!\ccalH)\big\|_2 \!\!\le\!\!C_\sigma\!\!\sum_{f\!=\!1}^{F}\! \| \tilde{\bbH}_{L}^{f}\bbx^f_{L-1}\!- \! \bbH_L^f\bbx_{L-1}^f\|_2 \!+\!C_\sigma\!\! \sum_{f\!=\!1}^{F}\! \| \tilde{\bbx}^f_{L\!-\!1}\! - \!\bbx^f_{L\!-\!1} \|_2.
\end{split}
\end{equation}

We now observe a recursion where the output difference of $\ell$th layer output is bounded by the output difference of $(\ell-1)$th layer output with an extra term (the first term in \eqref{eq:thm25}). Following the same process of \eqref{eq:thm21}-\eqref{eq:thm25}, we get
\begin{align} \label{eq:thm26}
&\| \tilde{\bbx}^f_{L\!-\!1}\!\! -\! \!\bbx^f_{L\!-\!1} \|_2 \!\!\le\!\! C_\sigma\!\!\!\sum_{g\!=\!1}^{F}\! \| \tilde{\bbH}_{L\!-\!1}^{fg}\bbx^g_{L\!-\!2}\!- \! \bbH_{L-1}^{fg}\bbx_{L-2}^g\|_2 \!+\! C_\sigma\!\!\sum_{g\!=\!1}^{F}\! \| \tilde{\bbx}^g_{L-2}\! - \!\bbx^g_{L-2} \|_2.
\end{align}
Substituting \eqref{eq:thm26} into \eqref{eq:thm25}, we get
\begin{align} \label{eq:thm27}
&\big\|\tilde{\bbPhi}(\bbx;\bbS,\!\ccalH) - \bbPhi(\bbx;\bbS,\!\ccalH)\big\|_2 \le C_\sigma\!\sum_{f=1}^{F} \| \tilde{\bbH}_{L}^{f}\bbx^f_{L-1}\!- \! \bbH_L^f\bbx_{L-1}^f\|_2\\
&+C_\sigma^2\sum_{f,g=1}^{F} \| \tilde{\bbH}_{L-1}^{fg}\bbx^g_{L-2}\!- \! \bbH_{L-1}^{fg}\bbx_{L-2}^g\|_2 + C_\sigma^2 F \sum_{g=1}^{F} \| \tilde{\bbx}^g_{L-2}- \bbx^g_{L-2} \|_2. \nonumber
\end{align}
Unrolling this recursion until the input layer yields
\begin{align} \label{eq:thm28}
&\big\|\tilde{\bbPhi}(\bbx;\!\bbS,\!\ccalH) \!-\! \bbPhi(\bbx;\!\bbS,\!\ccalH)\big\|_2 \!\le\! \frac{C_\sigma}{F}\sum_{f,g=1}^{F}\! \| \tilde{\bbH}_{L}^{f}\bbx^f_{L-1}\!- \! \bbH_L^f\bbx_{L-1}^f\|_2 \\
&+\!\!\! \sum_{\ell = 2}^{L-1} C_\sigma^{L\!+\!1\!-\!\ell}F^{L \!-\! 1 \!-\! \ell}\!\!\! \sum_{f,g=1}^{F}\! \| \tilde{\bbH}_{\ell}^{fg}\bbx^g_{\ell\!-\!1}\!\!- \!\! \bbH_{\ell}^{fg}\bbx_{\ell\!-\!1}^g\|_2 + C_\sigma^L F^{L\!-\!3}\!\!\! \sum_{f,g=1}^{F} \| \tilde{\bbH}_{1}^{f}\bbx_{0}\!-\! \bbH_1^f\bbx_{0} \|_2 \nonumber
\end{align}
where $\bbx_0 = \bbx$ is the input signal. From \eqref{eq:thm28} and the inequality of arithmetic and geometric means, the square of $\big\|\tilde{\bbPhi}(\bbx;\!\bbS,\!\ccalH) \!-\! \bbPhi(\bbx;\!\bbS,\!\ccalH)\big\|_2$ is bounded as
\begin{align} \label{eq:thm29}
&\big\|\tilde{\bbPhi}(\bbx;\!\bbS,\!\ccalH) \!-\! \bbPhi(\bbx;\!\bbS,\!\ccalH)\big\|^2_2 \le L C_\sigma^2\!\sum_{f,g=1}^{F} \| \tilde{\bbH}_{L}^{f}\bbx^f_{L-1}\!- \! \bbH_L^f\bbx_{L-1}^f\|_2^2 \\
&\!+\!\! L\!\! \sum_{\ell = 2}^{L-1}\!\!C_\sigma^{2L\!+\!2\!-\!2\ell}\!F^{2L \!-\! 2\ell} \!\!\!\sum_{f,g=1}^{F}\!\! \!\| \tilde{\bbH}_{\ell}^{fg}\bbx^g_{\ell\!-\!1}\!\!- \! \bbH_{\ell}^{fg}\bbx_{\ell\!-\!1}^g\|_2^2 \nonumber\!+\! L C_\sigma^{2L} \!F^{2L\!-\!4}\!\!\!\sum_{f,g=1}^{F}\!\!\! \| \tilde{\bbH}_{1}^{f}\bbx_{0}\!-\! \bbH_1^f\bbx_{0} \|_2^2.
\end{align}

We now consider $\| \tilde{\bbH}_{L}^{f}\bbx^f_{L-1}\!- \! \bbH_{L}^{f}\bbx_{L-1}^f\|_2^2$ in the first term of \eqref{eq:thm29}. Using the result of Theorem \ref{theorem:filterStability} [cf. \eqref{eq:FilterStability}], we have
\begin{equation} \label{eq:thm210}
\begin{split}
&\mathbb{E}\big[ \| \tilde{\bbH}_{L}^{f}\bbx^f_{L-1}\!- \! \bbH_{L}^{f}\bbx_{L-1}^f\|_2^2 \big] \le n\alpha C_L^2 (1\!-\!p) \| \bbx^f_{L-1} \|_2^2 + \ccalO((1-p)^2).
\end{split}
\end{equation}
For the square norm $\| \bbx^f_{L-1} \|_2^2$ in the bound of \eqref{eq:thm210}, we observe
\begin{equation}\label{eq:thm211}
\begin{split}
\| \bbx^f_{L-1} \|_2^2\!\le\! \| \sigma\big(\sum_{g=1}^F\!\bbu_{L-2}^{fg}\big) \|_2^2 \le C_\sigma^2 F\! \sum_{g=1}^F\! \| \bbu_{L-2}^{fg} \|_2^2 \!\le\! C_\sigma^2 F\! \sum_{g=1}^F\! \left\| \bbx_{L-2}^{g} \right\|_2^2 
\end{split}
\end{equation}          
where the Lipschitz condition of the nonlinearity, the triangular inequality and similar result as Lemma 3 in \cite{Zhan2020} are used. Following this process yields
\begin{equation}\label{eq:thm212}
\begin{split}
\| \bbx^f_{L-1} \|_2^2\le C_\sigma^{2L-2} F^{2L-4} \left\| \bbx \right\|_2^2 .
\end{split}
\end{equation}        
Substituting \eqref{eq:thm212} into \eqref{eq:thm210}, we get
\begin{align} \label{eq:thm213}
&\mathbb{E}\big[ \| \tilde{\bbH}_{L}^{f}\bbx^f_{L-1}- \bbH_{L}^{f}\bbx_{L-1}^f\|_2^2 \big] \\
& \le n\alpha C_L^2 C_\sigma^{2L-2} F^{2L-4} (1-p) \left\| \bbx \right\|_2^2 + \ccalO((1-p)^2).\nonumber
\end{align}
We bound $\| \tilde{\bbH}_{\ell}^{fg}\bbx^g_{\ell-1}\!- \! \bbH_{\ell}^{fg}\bbx_{\ell-1}^g\|_2^2$ and $\| \tilde{\bbH}_{1}^{f}\bbx^f_{0}- \bbH_1^f\bbx_{0}^f \|_2^2$ in the second and third terms of \eqref{eq:thm29} by following \eqref{eq:thm210}-\eqref{eq:thm213} since they have a similar form. Substituting these bounds into \eqref{eq:thm29} and using the linearity of the expectation in \eqref{eq:thm29}, we obtain
\begin{align}
&\mathbb{E}\big[ \big\|\tilde{\bbPhi}(\bbx;\!\bbS,\!\ccalH) \!-\! \bbPhi(\bbx;\!\bbS,\!\ccalH)\big\|^2_2\big] \!\!\le\! n \alpha C_L^2 L^2 C_\sigma^{2L} F^{2L - 2} (1\!-\!p) \| \bbx\|_2^2 \!+\! \ccalO((1\!-\!p)^2) \nonumber
\end{align}
completing the proof.
\end{proof}


\section{Proof of Corollary 1} \label{proof:corollary2}

\begin{proof}
Denote by $\Delta \bbPhi(\bbx;\bbS,\ccalH) := \big\|\tilde{\bbPhi}(\bbx;\bbS,\ccalH) - \bbPhi(\bbx;\bbS,\ccalH)\big\|^2_2$ the concise notation representing the square error of the GCNN induced by the stochastic perturbation [cf. \eqref{eq:GNNstability}]. By considering the conditional probability, we can represent the mean square error $\mathbb{E}\left[\Delta \bbPhi(\bbx;\bbS,\ccalH)\right]$ as
\begin{align}\label{eq:proofCoro2_1}
\mathbb{E}\left[\Delta \bbPhi(\bbx;\!\bbS,\!\ccalH)\right] &=\! \mathbb{E}\left[\Delta \bbPhi(\bbx;\!\bbS,\!\ccalH) \Big| \Delta \bbPhi(\bbx;\!\bbS,\!\ccalH) > \epsilon \right] \cdot \text{Pr} \left[ \Delta \bbPhi(\bbx;\!\bbS,\!\ccalH) > \epsilon \right]\\
&+\!\mathbb{E}\left[\Delta \bbPhi(\bbx;\!\bbS,\!\ccalH) \Big| \Delta \bbPhi(\bbx;\!\bbS,\!\ccalH) \le \epsilon \right] \cdot \text{Pr} \left[ \Delta \bbPhi(\bbx;\!\bbS,\!\ccalH) \le \epsilon \right]\nonumber\!.
\end{align}
From the fact that $\Delta \bbPhi(\bbx;\bbS,\ccalH) \ge 0$, we get
\begin{align}\label{eq:proofCoro2_2}
&\mathbb{E}\left[\Delta \bbPhi(\bbx;\bbS,\ccalH)\right] \\
&\ge 0 \cdot \text{Pr} \left[ \Delta \bbPhi(\bbx;\!\bbS,\!\ccalH) \!\le\! \epsilon \right] \!+\! \mathbb{E}\left[\Delta \bbPhi(\bbx;\!\bbS,\!\ccalH) \Big| \Delta \bbPhi(\bbx;\!\bbS,\!\ccalH) \!>\! \epsilon \right]\! \cdot \text{Pr} \left[ \Delta \bbPhi(\bbx;\!\bbS,\!\ccalH) \!>\! \epsilon \right]\nonumber\\
&\ge \epsilon \cdot \text{Pr} \left[ \Delta \bbPhi(\bbx;\bbS,\ccalH) > \epsilon \right] \nonumber
\end{align}
where the fact that $\mathbb{E}\left[\Delta \bbPhi(\bbx;\bbS,\ccalH) \Big| \Delta \bbPhi(\bbx;\bbS,\ccalH) > \epsilon \right] \ge \epsilon$ is used in the second inequality. Note that from Theorem \ref{theorem:GNNstability}, we have
\begin{equation}\label{eq:proofCoro2_3}
\mathbb{E}\left[\Delta \bbPhi(\bbx;\bbS,\ccalH)\right] \le C (1-p) \| \bbx \|^2_2 + \ccalO\big((1-p)^2\big)
\end{equation}
where $C$ is the stability constant in Theorem \ref{theorem:GNNstability} [cf. \eqref{eq:GNNstability}]. By substituting \eqref{eq:proofCoro2_3} into \eqref{eq:proofCoro2_2}, we get
\begin{align}\label{eq:proofCoro2_4}
\text{Pr} \left[ \Delta \bbPhi(\bbx;\bbS,\ccalH) > \epsilon \right] \le \frac{C (1-p) \| \bbx \|^2_2}{\epsilon} + \ccalO((1-p)^2).
\end{align}
Since $\text{Pr} \left[ \Delta \bbPhi(\bbx;\bbS,\ccalH) > \epsilon \right] + \text{Pr} \left[ \Delta \bbPhi(\bbx;\bbS,\ccalH) \le \epsilon \right] = 1$, we obtain
\begin{align}\label{eq:proofCoro2_5}
\text{Pr}\left[\big\|\tilde{\bbPhi}(\bbx;\!\bbS,\!\ccalH) \!-\! \bbPhi(\bbx;\!\bbS,\!\ccalH)\big\|^2_2 \le \epsilon \right] \ge 1- \frac{C(1-p)\|\bbx\|_2^2 }{\epsilon} - \ccalO((1-p)^2)
\end{align}
completing the proof.
\end{proof}


\bibliography{mybibfile,biblioOp}

\end{document}